\documentclass[Royal,sageh,times]{sagej}

\usepackage{moreverb,url}
\usepackage{xcolor}
\usepackage[colorlinks,bookmarksopen,bookmarksnumbered,citecolor=red,urlcolor=red]{hyperref}

\def\volumeyear{2026}

\usepackage{longtable}
\usepackage{booktabs}
\usepackage{siunitx}
\usepackage{multirow}
\usepackage{textcomp}
\usepackage{float}
\usepackage{amsmath}
\usepackage{amssymb}
\usepackage{amsthm}
\usepackage{tikz}
\usepackage{bussproofs}
\usepackage{listings}
\usetikzlibrary{shapes,arrows,positioning,fit,calc,backgrounds}
\usepackage{subfigure}

\newtheorem{definition}{Definition}[section]
\newtheorem{lemma}{Lemma}[section]
\newtheorem{theorem}{Theorem}[section]
\newtheorem{corollary}{Corollary}[section]

\begin{document}

\runninghead{Allen, Chhikara et al.}

\title{Sound and Complete Neurosymbolic Reasoning with LLM-Grounded Interpretations}

\author{Bradley P. Allen\affilnum{1}, Prateek Chhikara\affilnum{2}, Thomas Macaulay Ferguson\affilnum{3}, Filip Ilievski\affilnum{4} and Paul Groth\affilnum{1}}

\affiliation{\affilnum{1}University of Amsterdam, Amsterdam, The Netherlands\\
\affilnum{2}University of Southern California, Los Angeles, CA, United States\\
\affilnum{3}Rensselaer Polytechnic Institute, Troy, NY, United States\\
\affilnum{4}Vrije Universiteit Amsterdam, Amsterdam, The Netherlands}

\corrauth{Bradley P. Allen}

\email{b.p.allen@uva.nl}

\begin{abstract}
Large language models (LLMs) have demonstrated impressive capabilities in natural language understanding and generation, but exhibit problems with logical consistency in their output. How can we harness LLMs' broad-coverage parametric knowledge in formal reasoning despite their inconsistency? We present a method for directly integrating an LLM into the interpretation function of the formal semantics for a paraconsistent logic. We evaluate the method empirically using datasets derived from the short-form factuality benchmarks GPQA and SimpleQA, showing that bilateral factuality evaluation improves macro-F1 over a unilateral baseline by roughly 6 percentage points on both benchmarks (at the cost of reduced coverage, as abstention is triggered on inconsistent or uncertain cases). We further describe a proof-of-concept tableau reasoner implementing the method, and apply it to a medication-safety knowledge base of 228 asserted and 712 inferred statements: the system detects 92 gluts corresponding to medically significant errors (e.g., opioids inferred as non-addictive, beta-blockers inferred as safe in asthma) while remaining satisfiable, demonstrating that contradictions are localized rather than causing logical explosion. Unlike prior work, our method offers a theoretical framework with a practical implementation for neurosymbolic reasoning that leverages an LLM's knowledge while preserving the underlying logic's soundness and completeness properties.
\end{abstract}


\maketitle

\section{Introduction}
\label{sect:intro}

Applications involving commonsense reasoning and biomedical knowledge, where inconsistencies and incomplete information are commonplace, remain challenging frontiers for AI systems.
While LLMs encode vast parametric knowledge \citep{petroni2019language}, they also suffer from inconsistency and incompleteness \citep{cheng2025empowering}, limiting their use as knowledge bases for such applications.
Efforts to connect logical reasoning with LLMs relying on prompting strategies and external symbolic solvers \citep{wei2022chain,cheng2025empowering} show potential but exhibit significant shortcomings as well \citep{hoppe2025investigating}, lacking formal frameworks for managing LLM knowledge inconsistency and incompleteness. 

\textit{Paraconsistent logics} \citep{sep-logic-paraconsistent} are multi-valued non-classical logics that handle inconsistent information without logical explosion, where contradictions would make everything provable. \textit{Belnap computers} \citep{belnap1977how,belnap1977useful} are theoretical constructions described by Nuel Belnap to model contexts in which machines are responsible for reasoning in the face of incomplete or inconsistent information. 

We propose a Belnap computer using an \textit{LLM judge} as an external knowledge source (Figure \ref{fig:bcllmint}). LLM judges scale factuality evaluation of LLM output for short- or long-form question answering tasks, returning truth valuations for statements in the LLM's output \citep{li2024llms}. In our approach, an LLM judge responds to a query for the valuation of an atomic formula in the context of a paraconsistent reasoner with a \textit{generalized truth value} \citep{sep-truth-values,shramko2011truth}. Generalized truth values allow the LLM judge to provide information to the reasoner not only about the degree of truth of an atomic formula given the LLM's parametric knowledge, but also with respect to the degree of knowledge the LLM has about the formula.

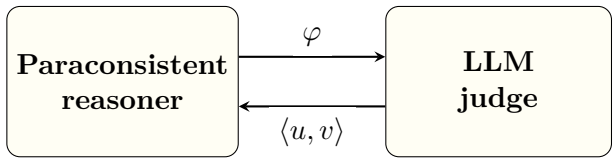
\begin{figure}[t]
\centering
\begin{tikzpicture}[
  scale=1, 
  transform shape,
  node distance=6cm,
  box/.style={rectangle, draw, text centered, minimum width=3cm, minimum height=2cm, fill=yellow!5, rounded corners=5pt},
  arrow/.style={thick,->,>=stealth},
  label/.style={midway, fill=white, font=\small}
]

\node[box, align=center] (reasoner) at (0,0) {\textbf{Paraconsistent}\\\textbf{reasoner}};

\node[box, align=center] (llm) at (5,0) {\textbf{LLM}\\\textbf{judge}};

\draw[arrow] ([yshift=0.33cm]reasoner.east) -- ([yshift=0.33cm]llm.west) node[midway, above, fill=white] {$\varphi$};
\draw[arrow] ([yshift=-0.33cm]llm.west) -- ([yshift=-0.33cm]reasoner.east) node[midway, below, fill=white] {$\langle u,v \rangle$};

\end{tikzpicture}
\caption{A Belnap computer using an LLM judge as a source of knowledge. Let $\mathcal{L}$ be the object language for a paraconsistent logic, and let $\mathcal{L}_{AT}$ be the set of atomic formulas. A paraconsistent reasoner (left) sends an atomic formula $\varphi \in \mathcal{L}_{AT}$ to the LLM judge (right), which returns a generalized truth value $\langle u,v \rangle$, such that $u$ indicates if the LLM judge was able to verify $\varphi$, and $v$ indicates if the LLM judge was able to refute $\varphi$.}
\label{fig:bcllmint}
\end{figure}

Section \ref{sect:related} discusses related work, and 
Section \ref{sect:llmgrounded} defines \textbf{\textit{a bilateral factuality evaluation function (contribution i)}}, which separately assesses whether an LLM can verify and whether it can refute a claim made by an atomic formula. Unlike standard factuality evaluation, which yields a single truth value, bilateral evaluation surfaces the LLM's epistemic state: consistent knowledge, inconsistency (where the LLM both affirms and denies), or uncertainty (where the LLM can do neither). This is valuable as an evaluation method in and of itself, providing actionable information beyond that provided by current LLM judge approaches to factuality evaluation. Section \ref{sec:preservation}
shows how bilateral evaluation integrates naturally into  the formal semantics of a paraconsistent logic through \textbf{\textit{an LLM-grounded interpretation (contribution ii)}} that preserves the soundness and completeness of analytic tableau systems for reasoning in the logic, enabling principled reasoning over LLM parametric knowledge despite its inherent inconsistency and incompleteness. Section \ref{sect:experiments} provides \textbf{\textit{empirical evidence for practical implementation of bilateral factuality evaluation (contribution iii)}}, presenting evaluation findings using benchmarks derived from short-form factuality benchmarks and discussing limitations. Section \ref{sect:tableau} describes \textbf{\textit{a demonstration of a scalable proof-of-concept implementation of a Belnap computer as a paraconsistent reasoner with LLM integration (contribution iv)}} as described in Section \ref{sect:llmgrounded} and \ref{sec:preservation}, and demonstrates our approach's feasibility and scalability. Section \ref{sect:limitations} discusses limitations of our approach, and Section \ref{sect:conclusion} concludes with a summary and discussion of future work.\footnote{We note that this article is an extended version of \cite{pmlr-v284-allen25a}; we express our deep appreciation to the two anonymous reviewers for valuable feedback which contributed to this version.}

\section{Related work}
\label{sect:related}

\paragraph{Logical reasoning with LLMs}
Current approaches to reasoning with LLMs
\citep{hoppe2025investigating,cheng2025empowering} appear in Figure \ref{fig:four_configurations}. In a \textit{prompt-based} approach (\textit{a}), an LLM is prompted with a verbalization $\delta(\Gamma) \in \Sigma^*$ of a set of formulas $\Gamma$, and performs natural language reasoning to produce a verbalization $\delta(\varphi)$ of $\varphi$ \citep{wei2022chain,kojima2022large,dhuliawala2023chain,yao2023tree}. In a \textit{solver-based} approach (\textit{b}), an LLM is prompted with a verbalization of a set of formulas and produces a set of formulas $\Gamma$ in an object language $\mathcal{L}$, which a reasoner uses to deduce $\varphi$ \citep{pan2023logic,olausson2023linc,callewaert2025verus}. In an approach based on \textit{pre-training or fine-tuning} (\textit{c}), a reasoner provides a training set of proofs $\Pi$, and the LLM learns from that to reason as in (\textit{a}) \citep{jiao2023exploring,morishita2024enhancing,feng2024language,liu2025synlogic}. Unlike approaches (\textit{b}) and (\textit{c}) that use LLMs alongside reasoning systems, in our proposed \textit{interpretation-based} approach (\textit{d}) we integrate LLMs directly into the formal semantics of the reasoner's logic itself, by using an LLM to implement an interpretation function $\mathcal{I}$ to be used by the reasoner in inferring $\varphi$. This allows us to provide formal guarantees about the soundness and completeness of the reasoning process that incorporates the LLM. 
In contrast with approaches (\textit{a}), (\textit{b}), and (\textit{c}), instead of trying to get an LLM to reason using logic, we get a logic to reason using an LLM. 

This architectural distinction has important methodological implications. Systems such as Logic-LM \citep{pan2023logic} and LINC \citep{olausson2023linc} fall within configurations (b) or (c): they use neural components to generate or transform logical representations, which are then processed by symbolic reasoners. DeepProbLog~\citep{manhaeve2018deepproblog} shares our interpretation-based architecture but operates within classical probabilistic semantics, which cannot accommodate the contradictions LLMs routinely produce. Our approach addresses a different problem than do these systems: how to formally integrate unreliable knowledge sources (in our case, LLMs) into reasoning systems while preserving soundness and completeness guarantees. The relevant comparison is not benchmark accuracy on clean reasoning tasks, but robustness to the inconsistencies that real-world knowledge sources exhibit. Our contribution is that bilateral evaluation surfaces the epistemic state of the LLM, and paraconsistent semantics provides principled machinery for reasoning over that state.

\begin{figure}[t]
\centering
\subfigure[Prompt-based]{%
\begin{tikzpicture}[
  scale=0.8, 
  transform shape,
  node distance=6cm,
  box/.style={rectangle, draw, text centered, minimum width=2cm, minimum height=1cm, fill=red!5, rounded corners=5pt},
  arrow/.style={thick,->,>=stealth},
  label/.style={midway, fill=white, font=\small}
]
\node[box, align=center] (llm) at (3.5,0) {\textbf{LLM}};
\draw[arrow] ([yshift=1cm]llm.north) -- (llm.north) node[midway, left] {$\delta(\Gamma)$};
\draw[arrow] (llm.south) -- ([yshift=-1cm]llm.south) node[midway, left] {$\delta(\{\varphi\})$};
\end{tikzpicture}
\label{fig:subfig_a}
}
\hspace{1cm} 
\subfigure[Solver-based]{%
\begin{tikzpicture}[
  scale=0.8, 
  transform shape,
  node distance=6cm,
  box/.style={rectangle, draw, text centered, minimum width=2cm, minimum height=1cm, fill=blue!5, rounded corners=5pt},
  arrow/.style={thick,->,>=stealth},
  label/.style={midway, fill=white, font=\small}
]
\node[box, align=center] (reasoner) at (0,0) {\textbf{Reasoner}};
\draw[arrow] (reasoner.south) -- ([yshift=-1cm]reasoner.south) node[midway, left] {$\varphi$};
\node[box, align=center] (llm) at (3.5,0) {\textbf{LLM}};
\draw[arrow] ([yshift=1cm]llm.north) -- (llm.north) node[midway, left] {$\delta(\Gamma)$};
\draw[arrow] (llm) -- (reasoner) node[midway, below, fill=white] {$\Gamma$};
\end{tikzpicture}
\label{fig:subfig_b}
}

\vspace{0.5cm} 

\subfigure[Pre-train/fine tune]{%
\begin{tikzpicture}[
  scale=0.8, 
  transform shape,
  node distance=6cm,
  box/.style={rectangle, draw, text centered, minimum width=2cm, minimum height=1cm, fill=green!5, rounded corners=5pt},
  arrow/.style={thick,->,>=stealth},
  label/.style={midway, fill=white, font=\small}
]
\node[box, align=center] (reasoner) at (0,0) {\textbf{Reasoner}};
\node[box, align=center] (llm) at (3.5,0) {\textbf{LLM}};
\draw[arrow] ([yshift=1cm]llm.north) -- (llm.north) node[midway, left] {$\delta(\Gamma)$};
\draw[arrow] (llm.south) -- ([yshift=-1cm]llm.south) node[midway, left] {$\delta(\{\varphi\})$};
\draw[arrow] (reasoner) -- (llm) node[midway, below, fill=white] {$\Pi$};
\end{tikzpicture}
\label{fig:subfig_c}
}
\hspace{1cm} 
\subfigure[Interpretation-based]{%
\begin{tikzpicture}[
  scale=0.8, 
  transform shape,
  node distance=6cm,
  box/.style={rectangle, draw, text centered, minimum width=2cm, minimum height=1cm, fill=yellow!5, rounded corners=5pt},
  arrow/.style={thick,->,>=stealth},
  label/.style={midway, fill=white, font=\small}
]
\node[box, align=center] (reasoner) at (0,0) {\textbf{Reasoner}};
\draw[arrow] ([yshift=1cm]reasoner.north) -- (reasoner.north) node[midway, left] {$\Gamma$};
\draw[arrow] (reasoner.south) -- ([yshift=-1cm]reasoner.south) node[midway, left] {$\varphi$};
\node[box, align=center] (llm) at (3.5,0) {\textbf{LLM}};
\draw[arrow] (llm) -- (reasoner) node[midway, below, fill=white] {$\mathcal{I}$};
\end{tikzpicture}
\label{fig:subfig_d}
}

\caption{Approaches to logical reasoning with LLMs. Let $\mathcal{L}$ be a first-order language, $\Gamma$ be a set of statements in $\mathcal{L}$, $\varphi$ be a statement in $\mathcal{L}$, $\mathcal{I}$ be an interpretation for $\mathcal{L}$, $\Pi$ be a set of proofs of statements in $\mathcal{L}$, and $\delta: \mathcal{P}(\mathcal{L}) \rightarrow \Sigma^*$ be a verbalization function that takes a set of formulas in $\mathcal{L}$ and returns a natural language translation of the formulas. In each approach, we show how reasoning is performed in the context of generating a formal or natural language proof showing that $\Gamma \vdash \varphi$.} 
\label{fig:four_configurations}
\end{figure}
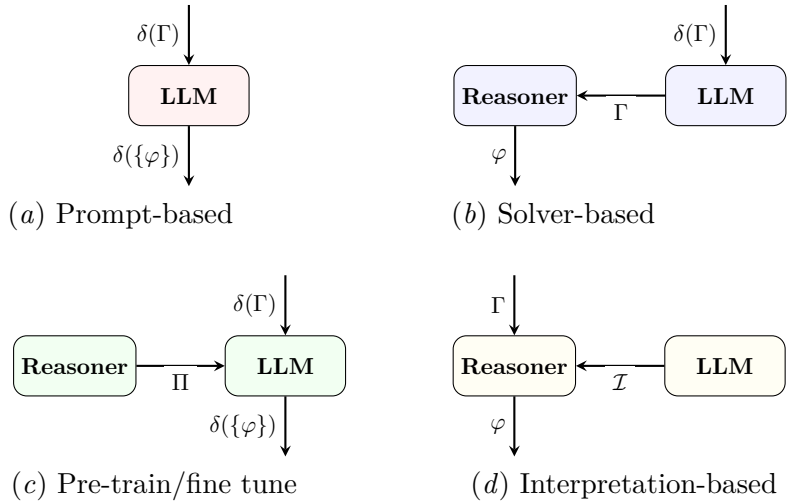

\paragraph{Factuality evaluation using LLM judges}
Factuality evaluation is the assessment of whether the output of a language model is factually correct \citep{wang2023survey,bang2025hallulens}. Recent work has focused on LLM judges \citep{zheng2023judging,zhu2023judgelm} as a means to scale factuality evaluation, by prompting an LLM to produce a truth value assignment of the output of an LLM that is being evaluated, specifically a short- \citep{wei2024measuring} or long-form \citep{jacovi2025facts} answer to a question. We apply LLM judges to assign generalized truth values to natural language translations of atomic formulas, taking a single-answer grading approach \citep{zheng2023judging}. In contrast with current approaches to factuality evaluation, the use of generalized truth values provides information as to the LLM's epistemic stance towards the formula in question.

\paragraph{Generalized truth values}

Generalized truth values offer significant advantages over truth valuations provided by current approaches to factuality evaluation using LLM judges, presenting a more sophisticated framework for handling complex evidential scenarios. These systems enable systematic distinctions between different types of evidence, as demonstrated by the theoretical work of \cite{shramko2011truth} and \cite{ferguson2021tableaux}, allowing evaluators to categorize and weight various forms of factual support rather than treating all evidence as equivalent. Additionally, generalized truth values provide principled methods for handling inconsistent and incomplete information, a critical advantage given the messy nature of real-world factual claims where evidence may be contradictory or absent, as explored in the foundational work by \cite{shramko2011truth}, the extensions by \cite{szmuc2019epistemic} and \cite{ferguson2021tableaux}, and the theoretical groundwork laid by \cite{correia2010grounding}. Perhaps, most importantly, for practical evaluation applications, these frameworks enhance the explanatory transparency of logical valuations through their structured nature, as noted by \cite{shramko2011truth}, \cite{ferguson2021tableaux}, and \cite{fine2016angellic}, providing clear logical foundations that make evaluation decisions more interpretable and allowing researchers to understand not just whether a claim is deemed factual, but why that determination was reached and what types of evidence contributed to the final assessment. Our work, following \cite{ferguson2021tableaux}, uses truth values that are elements of the bilattice $\mathcal{{NINE}}$ \citep{arieli1998value}.

\paragraph{Bilateralism}

Bilateralism in logic \citep{rumfitt2000yes} holds that understanding a proposition requires grasping both the conditions under which it can be asserted and the conditions under which it should be denied. Meaning isn't just about knowing when something is true, but also explicitly understanding when it is false. This philosophical view contrasts with unilateral approaches where only conditions for truth or assertion are primary, and falsity or denial is treated as derivative (just the absence or negation of truth). Bilateralists argue this misses something fundamental about meaning and inference, and that having explicit roles for both verification and refutation leads to better logical reasoning and clearer understanding. Our work empirically validates Rumfitt's philosophical position: treating assertion and denial as primitive speech acts (rather than reducing one to the other) provides superior analytical power. 

\paragraph{Multi-valued logics for paraconsistent reasoning}
Work that builds on Belnap's four-valued semantics has led to the development of a range of non-classical paraconsistent logics. \cite{patel1989four} showed how this idea could be applied to make terminological logics capable of performing subsumption correctly in the presence of contradictory knowledge. \cite{kamide2010paraconsistent}, \cite{ma2007algorithms}, and \cite{maier2013paraconsistent} expanded on this work to define a number of paraconsistent description logics. More recently, \cite{ferguson2017faulty} has proposed a computational interpretation of versions of first degree entailment ($\mathbf{FDE}$) and Richard Angell's logic of analytic containment ($\mathbf{AC}$). This interpretation models reasoners using these logics as Belnap computers, and leads to a formal framework for $\mathbf{FDE}$ and $\mathbf{AC}$ as bilateral logics with sound and complete analytic tableau systems. We show how an LLM judge can be used to provide an interpretation for $\mathbf{AC}$ that preserves the soundness and completeness of the tableau system $\mathbf{ACrQ}$, with applications to description logics as described in \citep{ferguson2021modeling}.

\paragraph{Interpreting symbols in logical statements as natural language terms}
Previous efforts have explored the degree to which meaning can be captured through the use of natural language terms as symbols in a formal knowledge representation language. 
Google distance between symbol names has been used to weight ontology matches, explicitly leveraging that ontology concepts are often named with meaningful words \citep{gligorov2007using}.
Semantic distances between natural language representations of concepts have been used to reason with inconsistent ontologies, selecting maximally consistent subsets based on linguistic similarity \citep{huang2008using}. The degree to which symbol names carry social meaning in knowledge graphs has been quantified, providing an information-theoretic framework for measuring the effectiveness of this symbol-language correspondence \citep{de2016names}. Most recently, work on context-aware clause selection in theorem proving has shown that symbol names can provide valuable heuristics even in purely formal reasoning tasks \citep{schon2025context}. Our approach builds on these efforts, using the linguistic processing capabilities of LLMs to interpret the predicate and constant symbols in atomic formulas in $\mathbf{AC}$ as terms in natural language. At first glance, this may seem to fly in the face of traditional approaches to logic; however, traditional formal systems have always required human judgment to determine what constitutes truth for atomic formulas. Our contribution makes this interpretative act explicit by using LLMs to operationalize how humans assess claims through language.

\section{Bilateral factuality evaluation of an atomic formula using an LLM}
\label{sect:llmgrounded}

$\mathbf{AC}$ is a \textit{conceptivist} logic \citep{ferguson2017computational} that addresses hierarchical relationships between concepts. While in this work we focus on $\mathbf{AC}$ as a paraconsistent logic, it is also \textit{paracomplete}, i.e., it rejects the law of the excluded middle. The combination of those two properties makes it particularly suitable for applications involving vague predicates, incomplete information, or situations where classical logic's demands for both consistency and completeness are too strong. Appendix \ref{apd:ac} provides a definition of $\mathbf{AC}$ with restricted quantification, which is necessary to support concept subsumption and existential quantification of roles when used as a description logic \citep{ferguson2021tableaux}. This definition treats $\mathbf{AC}$ as a \textit{bilateral} logic, i.e., a logic which manages values for both the truth and falsity of a formula separately. Bilateralism in philosophical logic \citep{rumfitt2000yes} holds that understanding a proposition requires grasping both the conditions under which it can be asserted and the conditions under which it should be denied. We operationalize this principle by evaluating the factuality of atomic formulas using an LLM judge. 

First, we generate a natural language verbalization of an atomic formula, then prompt the LLM to generate two statements on the verifiability and refutability of the assertion. The statements are then mapped into the set of truth values $\mathcal{V}_3$ used in weak Kleene logic \citep{kleene1952introduction,szmuc2019epistemic}, i.e., $\mathfrak{t}$ (true), $\mathfrak{e}$ (undefined), and $\mathfrak{f}$ (false). 
Weak Kleene truth values allow us to formalize the semantics of LLM-judge output. For example, the SimpleQA grader used in the SimpleQA benchmark \citep{wei2024measuring} grades an LLM's answer to a question as either ``CORRECT", ``INCORRECT", or ``NOT ATTEMPTED"; the PreciseWikiQA Question Answerability Prompt in the HalluLens benchmark \citep{bang2025hallulens} uses ``UNVERIFIABLE" instead of ``NOT ATTEMPTED". We equate ``CORRECT" with $\mathfrak{t}$ and ``INCORRECT" with $\mathfrak{f}$; equating $\mathfrak{e}$  with ``NOT ATTEMPTED" or ``UNVERIFIABLE" is consistent with Kleene's original statement that it indicates ``an absence of information" that a given formula is either $\mathfrak{t}$ or $\mathfrak{f}$ \citep[p. 333]{kleene1952introduction}. 

Finally, we pair the weak Kleene truth value $u$ for verifiability with the weak Kleene truth value $v$ for refutability to yield a generalized truth value $\langle u, v \rangle \in \mathcal{V}_3 \times \mathcal{V}_3$. Generalized truth values offer significant advantages over truth values provided by current approaches to factuality evaluation using LLM judges: they enable systematic distinctions between different types of evidence \citep{shramko2011truth,ferguson2021tableaux}, provide principled methods for handling inconsistent and incomplete information \citep{shramko2011truth,szmuc2019epistemic,ferguson2021tableaux,correia2010grounding}, and enhance the explanatory transparency of logical valuations through their structured nature \citep{shramko2011truth,ferguson2021tableaux,fine2016angellic}. Figure \ref{fig:discovery} shows an example of this process in action; a longer example is provided in Appendix \ref{apd:examples}.

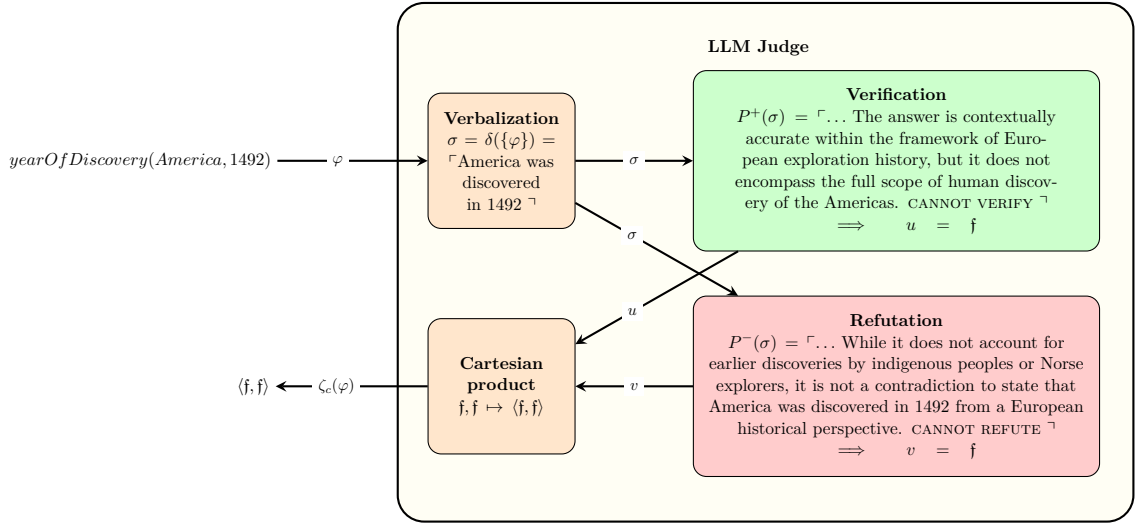
\begin{figure}[t]
\centering
\resizebox{\textwidth}{!}{%
\begin{tikzpicture}[
  node distance=4cm,
  box/.style={rectangle, draw, text centered, minimum width=3cm, minimum height=2cm, rounded corners=5pt},
  smallbox/.style={rectangle, draw, text centered, minimum width=3cm, minimum height=3cm, rounded corners=5pt, text width=3cm, text badly ragged},
  promptbox/.style={rectangle, draw, text centered, minimum width=9cm, minimum height=4cm, rounded corners=5pt, fill=blue!20, text width=8.5cm, text badly ragged},
  bigbox/.style={rectangle, draw, dashed, rounded corners=10pt, thick, fill=yellow!5, inner sep=15pt},
  graybox/.style={smallbox, fill=gray!20},
  greenbox/.style={promptbox, fill=green!20},
  redbox/.style={promptbox, fill=red!20},
  orangebox/.style={smallbox, fill=orange!20},
  arrow/.style={thick,->,>=stealth},
  label/.style={midway, fill=white, font=\small}
]

\begin{scope}
    \draw[rounded corners=10pt, thick, fill=yellow!5] 
         (-15.8,-5.5) rectangle (0.5,6)
         node[anchor=north west, xshift=5pt, yshift=-5pt] at (-9.25,5.5) {\textbf{LLM Judge}};
\end{scope}

\node[align=center] (phi) at (-21.5,2.5) {$yearOfDiscovery(America,1492)$};

\node[orangebox, align=center] (sigma) at (-13.5,2.5) {\textbf{Verbalization} \\ $\sigma = \delta(\{\varphi\}) =$ $\ulcorner$America was discovered in 1492 $\urcorner$};

\node[greenbox, align=center] (verification) at (-4.75,2.5) {\textbf{Verification} \\ $P^+(\sigma) = \ulcorner \ldots$ The answer is contextually accurate within the framework of European exploration history, but it does not encompass the full scope of human discovery of the Americas. \textsc{cannot verify} $\urcorner$ \\ $\implies u = \mathfrak{f}$};

\node[redbox, align=center] (refutation) at (-4.75,-2.5) {\textbf{Refutation} \\ $P^-(\sigma) = \ulcorner \ldots$ While it does not account for earlier discoveries by indigenous peoples or Norse explorers, it is not a contradiction to state that America was discovered in 1492 from a European historical perspective. \textsc{cannot refute} $\urcorner$ \\ $\implies v = \mathfrak{f}$};

\node[orangebox, align=center] (uv) at (-13.5,-2.5) {\textbf{Cartesian} \\ \textbf{product} \\ $\mathfrak{f}, \mathfrak{f} \mapsto \langle \mathfrak{f}, \mathfrak{f} \rangle$};

\node[align=center] (v1) at (-19,-2.5) {$\langle \mathfrak{f}, \mathfrak{f} \rangle$};

\draw[arrow] (phi) -- (sigma) node[label,pos=0.4] {$\varphi$};
\draw[arrow] (sigma) -- (verification) node[label,pos=0.5] {$\sigma$};
\draw[arrow] (sigma) -- (refutation) node[label,pos=0.35] {$\sigma$}; 
\draw[arrow] (verification) -- (uv) node[label,pos=0.65] {$u$};
\draw[arrow] (refutation) -- (uv) node[label,pos=0.5] {$v$};
\draw[arrow] (uv) -- (v1) node[label, pos=0.6] { $\zeta_c(\varphi)$};

\end{tikzpicture}%
}
\caption{An example of bilateral factuality evaluation $\zeta_c$ as performed by the LLM judge shown in Figure \ref{fig:bcllmint}. Let $\varphi \in \mathcal{L}_{AT}$ be an assertion that the discovery of America occurred in the year 1492. $\sigma = \delta(\{\varphi\})$ is the verbalization of $\varphi$ (Definition \ref{def:verbalizationfn}). The bilateral factuality evaluation of $\varphi$ by an LLM judge (Definitions \ref{def:verificationfn}, \ref{def:refutationfn}, and \ref{def:zeta}) generates the truth value $\zeta_c(\varphi) = \langle \mathfrak{f}, \mathfrak{f} \rangle$. The LLM has in effect identified \textit{incompleteness} in its knowledge --- based on differing perspectives on who discovered America, it can neither verify nor refute $\varphi$.}
\label{fig:discovery}
\end{figure}

\paragraph{Preliminaries}
Let $\Sigma$ be a countable set of tokens and $\Sigma^*$ be the set of finite sequences of tokens $\ulcorner t_0 \ldots t_k \urcorner$, where $t_{0 \leq i \leq k} \in \Sigma$, $k \in \mathbb{N}$. 
Given sequences $\sigma, \sigma' \in \Sigma^*$, $\sigma \prec \sigma'$ iff $\sigma$ is a \emph{proper contiguous subsequence} of $\sigma'$. Let $L_\mathfrak{C}$ be an LLM trained on a corpus $\mathfrak{C} \in \mathcal{P}(\Sigma^*)$.

\begin{definition}
\label{def:verbalizationfn}
A \emph{verbalization function} $\delta: \mathcal{P}(\mathcal{L}) \rightarrow \Sigma^*$ is a total function that maps a set of formulas to a sequence of tokens. 
\end{definition}

The verbalization function $\delta$ thus serves as a critical bridge between formal logic and natural language in our framework, and makes the assumption that symbols in our logical language can meaningfully correspond to words or phrases that an LLM can interpret. In practice, $\delta$ can be implemented through three distinct approaches:
\begin{enumerate}
\item \textbf{Direct syntactic mapping}: Logical formulas are provided directly in their formal syntax (e.g., \texttt{yearOfDiscovery(America, 1492)}). This assumes the LLM can parse and understand logical notation. This is the approach used in the implementation of bilateral factuality evaluation used in the evaluation in Section \ref{sect:experiments} and the tableau reasoning system described in Section \ref{sect:tableau}.
\item \textbf{Template-based transformation}: Formulas are converted using predefined templates (e.g., ``America was discovered in 1492''). This approach provides more natural language but requires manual template creation for each predicate type.
\item \textbf{LLM-generated verbalization}: An LLM generates the natural language representation of the formula. This is the most flexible approach, but it introduces another layer of potential error and interpretation.
\end{enumerate}

\begin{definition}
\label{def:verificationfn}
A \emph{verification function} $P^+: \Sigma^* \rightarrow \Sigma^*$ prompts $L_\mathfrak{C}$ to take a verbalization of an atomic formula $\varphi \in \mathcal{L}_{AT}$ and generate a token sequence $\sigma^+$ that states that $\varphi$ is verified or that it cannot be verified.
\end{definition}

\begin{definition}
\label{def:refutationfn}
A \emph{refutation function} $P^-: \Sigma^* \rightarrow \Sigma^*$ prompts $L_\mathfrak{C}$ to take a verbalization of an atomic formula $\varphi \in \mathcal{L}_{AT}$ and generate a token sequence $\sigma^-$ that states that $\varphi$ is refuted or that it cannot be refuted.
\end{definition}

\begin{definition}
\label{def:zeta}
A \emph{bilateral factuality evaluation function} $\zeta : \mathcal{L}_{AT} \rightarrow \mathcal{V}_3 \times \mathcal{V}_3$ is a total function that given an atomic formula $\varphi \in \mathcal{L}_{AT}$ yields a pair $\langle u, v \rangle$ where:
$$u = \begin{cases}
\mathfrak{t} & \text{if } \ulcorner \mathsf{VERIFIED} \urcorner \prec P^+(\delta(\{\varphi\})) \\
\mathfrak{f} & \text{if } \ulcorner \mathsf{CANNOT\ VERIFY} \urcorner \prec P^+(\delta(\{\varphi\})) \\
\mathfrak{e} & \text{otherwise}
\end{cases}$$
$$v = \begin{cases}
\mathfrak{t} & \text{if } \ulcorner \mathsf{REFUTED} \urcorner \prec P^-(\delta(\{\varphi\})) \\
\mathfrak{f} & \text{if } \ulcorner \mathsf{CANNOT\ REFUTE} \urcorner \prec P^-(\delta(\{\varphi\})) \\
\mathfrak{e} & \text{otherwise}
\end{cases}$$
\end{definition}
 
The ``otherwise'' cases in Definition \ref{def:zeta} reflect LLMs failing to output tokens indicating verification or refutation state, e.g., by failing to
follow instructions or timing out during API calls. We use repeated sampling with majority vote \citep{brown2024large} to determine truth value components; other hallucination mitigation approaches, such as chain-of-verification \citep{dhuliawala2023chain}, are also admissible.

The definition of $\zeta$ leaves open the possibility that multiple calls over time might return different truth values. This violates the assumption of analytic tableau reasoning that atomic formulas have stable truth values within the scope of the reasoning process. We ensure this by using a caching version of the bilateral evaluation function $\zeta_c$ where valuations for atomic formulas are persistently and immutably stored. This type of caching is consistent with the range of optimization techniques used in description logics tableau reasoners \citep{gore2007exptime,nguyen2009efficient}.  

\begin{definition}
\label{def:zetasubc}
A \emph{caching bilateral factuality evaluation function} $\zeta_{c}: \mathcal{L}_{AT} \rightarrow \mathcal{V}_3 \times \mathcal{V}_3$ is a total function defined as:
$$\zeta_c(\varphi) =
\begin{cases}
    c(\varphi) & \text{if } \varphi \in \text{dom}(c) \\
    \zeta(\varphi) & \text{otherwise, and } c := c \cup \{(\varphi, \zeta(\varphi))\}
\end{cases}
$$
where $c$ is a persistent and immutable cache mapping atomic formulas to truth value pairs, and $dom(c)$ denotes the domain of $c$, i.e., the set of atomic formulas for which a valuation has already been stored: $\operatorname{dom}(c) = \{ \varphi \in L_{AT} \mid \exists \langle u,v\rangle \in \mathcal{V}_3 \times \mathcal{V}_3 \text{ such that } (\varphi, \langle u,v\rangle) \in c \}$.
\end{definition}

Having established a bilateral evaluation, we now show how this enables LLM judges to implement interpretation functions directly.

\section{LLM-grounded interpretations}
\label{sec:preservation}

We formalize how the bilateral evaluation function $\zeta_c$ is integrated into the definition of an interpretation for $\mathbf{AC}$, show that for every LLM-grounded $\mathbf{AC}$ interpretation there is an equivalent standard $\mathbf{AC}$ interpretation, and then show that soundness and completeness of the tableau-style analytic calculus $\mathbf{ACrQ}$ defined in Definition 18 of \cite{ferguson2021tableaux} is preserved when we adopt an LLM-grounded interpretation. 

These results are metalogical: they concern the reasoner's treatment of atomic valuations, not the quality of the valuations themselves, which we address empirically in Section \ref{sect:experiments}. They assume stability (guaranteed via caching; Definition \ref{def:zetasubc}, Lemma \ref{lem:stability}) and faithful verbalization of atomic formulas by $\delta$ (Section \ref{sect:limitations}). Throughout this section we write $\Gamma \models_{\mathcal{I}} \varphi$ to denote validity relative to the class of interpretations over which $\mathcal{I}$ is understood to range (either all standard $\mathbf{AC}$ interpretations or all LLM-grounded $\mathbf{AC}$ interpretations, as disambiguated by context), rather than model-checking in a single interpretation.

\begin{definition}
\label{def:llmgrounded}
    An \emph{LLM-grounded $\mathbf{AC}$ interpretation} $\mathcal{I} = \langle \mathbf{C}^\mathcal{I}, \mathbf{R}^\mathcal{I} \rangle$ is an $\mathbf{AC}$ interpretation such that for every function $R^\mathcal{I} \in \mathbf{R}^\mathcal{I}$ and $c_1^\mathcal{I}, \ldots, c_n^\mathcal{I} \in \mathbf{C}^\mathcal{I}$:
\begin{align*}
    R^\mathcal{I}(c_1^\mathcal{I}, \ldots, c_n^\mathcal{I}) = \zeta_c(R(c_1, \ldots, c_n))
\end{align*}
\end{definition}

\begin{lemma}[Stability of LLM-grounded interpretations]
\label{lem:stability}
For any LLM-grounded $\mathbf{AC}$ interpretation $\mathcal{I}$ and atomic formula $\varphi \in \mathcal{L}_{AT}$:
\begin{enumerate}
    \item $\mathcal{I}(\varphi)$ is well-defined and yields exactly one pair $\langle u, v \rangle \in \mathcal{V}_3 \times \mathcal{V}_3$
    \item Once computed, $\mathcal{I}(\varphi)$ remains constant throughout the reasoning process
\end{enumerate}
\end{lemma}

\begin{proof}
    The stability follows directly from Definition \ref{def:zetasubc} of $\zeta_c$. Let $t_0$ be the time at which $\zeta$ is first called to set $c(\varphi)$. If $\zeta_c(\varphi) = \langle u, v \rangle$ at time $t_0$, then for all subsequent calls at $t > t_0$, $\zeta_c(\varphi) = \langle u, v \rangle$. This ensures that once an atomic formula $\varphi$ has been evaluated and the returned pair $\langle u, v \rangle \in \mathcal{V}_3 \times \mathcal{V}_3$ has been cached, all subsequent evaluations will return the same pair from the cache, guaranteeing stability.
\end{proof}

Lemma \ref{lem:stability} states that the logical validity of derivations depends on logical operator structure rather than atomic proposition content, a fundamental principle in formal logic. The tableau method only depends on the truth-functional behavior of the logical connectives, which remains unchanged between standard and LLM-grounded interpretations.

\begin{lemma}[LLM-grounded to standard interpretation mapping]
\label{lem:groundedtostd}
    For any LLM-grounded $\mathbf{AC}$ interpretation $\mathcal{I} = \langle \mathbf{C}^\mathcal{I}, \mathbf{R}^\mathcal{I} \rangle$, there exists a standard $\mathbf{AC}$ interpretation $\mathcal{I}'$ that preserves the semantic behavior of $\mathcal{I}$ on all formulas.
\end{lemma}

\begin{proof}
    Given an LLM-grounded $\mathbf{AC}$ interpretation $\mathcal{I} = \langle \mathbf{C}^\mathcal{I}, \mathbf{R}^\mathcal{I} \rangle$, we define a standard $\mathbf{AC}$ interpretation $\mathcal{I}' = \langle \mathbf{C}^{\mathcal{I}'}, \mathbf{R}^{\mathcal{I}'} \rangle$ such that:
    \begin{enumerate}
        \item $\mathbf{C}^{\mathcal{I}'} = \mathbf{C}^\mathcal{I}$
        \item $\mathbf{R}^{\mathcal{I}'} = \mathbf{R}^\mathcal{I}$
        \item For all $R \in \mathbf{R} \text{, } R^{\mathcal{I}'}(c_1^{\mathcal{I}'}, \ldots, c_n^{\mathcal{I}'}) = R^\mathcal{I}(c_1^\mathcal{I}, \ldots, c_n^\mathcal{I})$
        \item For all $c \in \mathcal{C}' \text{, } c^{\mathcal{I}'} = c^\mathcal{I}$
    \end{enumerate}
    We show that for any formula $\varphi \in \mathcal{L}$, $\mathcal{I}(\varphi) = \mathcal{I}'(\varphi)$ by induction on the complexity of $\varphi$:
    \begin{itemize}
        \item For $\varphi = R(c_1, \ldots c_n) \in \mathcal{L}_{AT}$, then by Definition \ref{def:llmgrounded}, $\mathcal{I}(\varphi) = \mathcal{I}(R(c_1, \ldots, c_n)) = R^\mathcal{I}(c^\mathcal{I}_1, \ldots, c^\mathcal{I}_n) = R^{\mathcal{I}'}(c^{\mathcal{I}'}_1, \ldots, c^{\mathcal{I}'}_n) = \mathcal{I}'(R(c_1, \ldots, c_n)) = \mathcal{I}'(\varphi)$.
        \item For $\varphi = \neg \psi$, by Definition \ref{def:stdmap}, $\mathcal{I}(\neg \psi) = \langle \mathcal{I}_1(\psi), \mathcal{I}_0(\psi) \rangle$ and $\mathcal{I}'(\neg \psi) = \langle \mathcal{I'}_1(\psi), \mathcal{I'}_0(\psi) \rangle$. By the inductive hypothesis, $\mathcal{I}(\psi) = \mathcal{I}'(\psi)$. Therefore $\mathcal{I}(\neg \psi) = \mathcal{I}'(\neg \psi)$.
        \item For $\varphi = \psi \land \chi$, by Definition \ref{def:stdmap}, $\mathcal{I}(\psi \land \chi) = \langle \mathcal{I}_0(\psi) \, \dot{\land} \, \mathcal{I}_0(\chi), \mathcal{I}_1(\psi) \, \dot{\lor} \, \mathcal{I}_1(\chi) \rangle$ and similarly for $\mathcal{I}'$. By the inductive hypothesis, $\mathcal{I}(\psi) = \mathcal{I}'(\psi)$ and $\mathcal{I}(\chi) = \mathcal{I}'(\chi)$. Therefore $\mathcal{I}(\psi \land \chi) = \mathcal{I}'(\psi \land \chi)$.
    \end{itemize}
    The same arguments apply in the cases of disjunction, restricted universal quantification, and restricted existential quantification, again following Definition \ref{def:stdmap} in Appendix \ref{apd:ac}.
\end{proof}

\begin{lemma}[Standard to LLM-grounded interpretation mapping]
\label{lem:stdtogrounded}
    For any standard $\mathbf{AC}$ interpretation $\mathcal{I} = \langle \mathbf{C}^\mathcal{I}, \mathbf{R}^\mathcal{I} \rangle$, there exists an LLM-grounded $\mathbf{AC}$ interpretation $\mathcal{I}'$ that preserves the semantic behavior of $\mathcal{I}$ on all formulas.
\end{lemma}

\begin{proof}
    Let $\mathcal{I}=\langle\mathbf{C}^{\mathcal{I}},\mathbf{R}^{\mathcal{I}}\rangle$ be a standard $\mathbf{AC}$ interpretation. Define a key-value store $c_{\mathcal{I}}$ such that for all atomic $R(c_1,...,c_n)\in \mathcal{L}_{AT}$, $c_{\mathcal{I}}(R(c_1,...,c_n))=R^{\mathcal{I}}(c_1^{\mathcal{I}},...,c_n^{\mathcal{I}})$. 
    Since $c_{\mathcal{I}}$ is defined on every atomic formula, $\operatorname{dom}(c_{\mathcal{I}}) = \mathcal{L}_{AT}$, and so by Definition \ref{def:zetasubc}, $\zeta_{c_{\mathcal{I}}}$ returns $c_{\mathcal{I}}(\varphi)$ directly for each $\varphi \in \mathcal{L}_{AT}$, without invoking $\zeta$. 
    The induced LLM-grounded $\mathbf{AC}$ interpretation $\mathcal{I}'$ therefore agrees with $\mathcal{I}$ on all atoms by construction. An induction on formula complexity, parallel to that in Lemma \ref{lem:groundedtostd}, extends this agreement to all formulas.
\end{proof}

\begin{corollary}[Extensional equivalence of validity]
\label{cor:extensional}
    $\Gamma\vDash_{\mathbf{AC}}\varphi$ holds if and only if the inference from $\Gamma$ to $\varphi$ is valid over all LLM-grounded $\mathbf{AC}$ interpretations.
\end{corollary}

\begin{proof}
    Immediate from Lemmas \ref{lem:groundedtostd} and \ref{lem:stdtogrounded}. This extensional equivalence is what allows the metalogical properties of $\mathbf{ACrQ}$ to transfer to the LLM-grounded setting, as the following two theorems make explicit.
\end{proof}

\begin{theorem}[Preservation of soundness]
\label{thm:soundness}
  Let \(\Gamma\) be a finite set of formulas and \(\varphi\) a formula in 
  \(\mathbf{AC}\). If $\Gamma \vdash_{\mathbf{ACrQ}} \varphi$, then $\Gamma \models_{\mathcal{I}} \varphi$ is valid for all LLM-grounded $\mathbf{AC}$ interpretations $\mathcal{I}$. 
\end{theorem}

\begin{proof}
By Theorem 3 of \cite{ferguson2021tableaux}, if $\Gamma \vdash_{\mathbf{ACrQ}} \varphi$, then $\Gamma \models_{\mathbf{AC}} \varphi$. Therefore $\Gamma \models_{\mathcal{I}} \varphi$, as is the case for any $\mathbf{AC}$ interpretation, standard or LLM-grounded.
\end{proof}

\begin{theorem}[Preservation of completeness]
\label{thm:completeness}
  Let \(\Gamma\) be a finite set of formulas and \(\varphi\) a formula in 
  \(\mathbf{AC}\). Then if for all LLM-grounded $\mathbf{AC}$ interpretations $\mathcal{I}$, $\Gamma \models_{\mathcal{I}} \varphi$, then $\Gamma \vdash_{\mathbf{ACrQ}} \varphi$.
\end{theorem}

\begin{proof}
    Let $\mathcal{I}'$ be an arbitrary standard $\mathbf{AC}$ interpretation. By Lemma \ref{lem:stdtogrounded}, there exists an LLM-grounded $\mathbf{AC}$ interpretation $\mathcal{I}$ that agrees with $\mathcal{I}'$ on all formulas. By hypothesis, $\Gamma \models_{\mathcal{I}} \varphi$, and therefore $\Gamma \models_{\mathcal{I}'} \varphi$. Since $\mathcal{I}'$ was arbitrary, $\Gamma \models_{\mathcal{I}'} \varphi$ holds for every standard $\mathbf{AC}$ interpretation. Then by Theorem 4 of \cite{ferguson2021tableaux}, $\Gamma \vdash_{\mathbf{ACrQ}} \varphi$.
\end{proof}

Section \ref{sect:tableau} presents our implementation of these theoretical concepts in a complete tableau reasoning system, demonstrating the practical realizability of LLM-grounded interpretations.

\section{Factuality evaluation of LLM-grounded interpretations}
\label{sect:experiments}

To validate the factuality of LLM-grounded interpretations, we evaluate the bilateral evaluation function $\zeta$ that underlies them, using question/answer pairs from two short-form factuality benchmarks. Given our theoretical results above, this focus on the practicality of atomic formula valuation within our framework provides a proof-of-concept that a Belnap computer of the type described above is feasible.\footnote{Code and data for the experiment is available at \url{https://github.com/bradleypallen/bilateral-factuality-evaluation}.} 

\paragraph{Data and metrics} Question/answer pairs in short-form factuality benchmarks are typically factoids providing a question together with a short answer (Figure \ref{fig:discovery}). We use these as an approximation to the verbalization $\delta(\varphi)$ of an atomic formula $\varphi$. We used the short-form factuality benchmarks GPQA \citep{rein2023gpqa} and SimpleQA \citep{wei2024measuring} to create two test datasets (each with N=400) for our experiments, each balanced between positive and negative examples. Test data preparation and experimental setup are discussed in Appendix \ref{apd:experiments}. We evaluated the performance of $\zeta$ over the two datasets using two metrics: \textit{macro F1} against question/answer pairs that the judge did not abstain from, i.e., where the judge provided a valuation of $\zeta(\varphi) = \langle \mathfrak{t}, \mathfrak{f} \rangle$ (i.e., verified and not refuted) or $\zeta(\varphi) = \langle \mathfrak{f}, \mathfrak{t} \rangle$ (i.e., not verified and refuted), and \textit{coverage}, which is the percentage of the total set of question/answer pairs where the judge did not abstain. We also measured the \textit{time taken per evaluation}, and the \textit{number of tokens used per evaluation}.

\paragraph{LLM judges} We used three flagship LLMs (Llama 4 Maverick, GPT-4o, and Claude 3.5 Sonnet), and three distilled LLMs (Llama 4 Scout, GPT-4o Mini, and Claude 3.5 Haiku). Each LLM was evaluated using three different pairs of prompts: direct prompts that asked for a verification or refutation for the QA pair, zero-shot chain-of-thought prompts, and few-shot chain-of-thought prompts. As a baseline, we also used the six models and three prompt types to perform \textit{unilateral} factuality evaluation, which asks an LLM to simply determine whether a question/answer pair is $\mathfrak{t}$ or $\mathfrak{f}$. The prompt templates used are provided in Appendix \ref{apd:templates}. Standard errors (in parentheses) presented in tables in this section and in Appendix \ref{apd:experiments} were estimated by bootstrap resampling: 1000 subsamples of size N=100 were drawn from the classification results within each model category \citep{politis1994large}. 

\begin{table}[t]
\centering
\footnotesize
\begin{tabular}{llrrrrrrrr}
\toprule
& & \multicolumn{2}{c}{\textbf{Bilateral} ($\zeta$)} & \multicolumn{2}{c}{\textbf{Unilateral}} \\
\cmidrule(lr){3-4} \cmidrule(lr){5-6}
\textbf{Dataset} & \textbf{Model Type} & \textbf{Macro F1} & \textbf{Coverage} & \textbf{Macro F1} & \textbf{Coverage} \\
\midrule
GPQA & flagship & 0.699 (0.010) & 0.589 (0.008) & 0.633 (0.007)	& 1.000 (0.000) \\
& distilled & 0.608 (0.011) & 0.504 (0.008) & 0.559 (0.008) & 1.000 (0.000) \\
\midrule
SimpleQA & flagship & 0.736 (0.009) & 0.584 (0.008) & 0.657 (0.008) & 1.000 (0.000)\\
& distilled & 0.624 (0.011) & 0.499 (0.008) & 0.570 (0.008) & 1.000 (0.000) \\
\bottomrule
\end{tabular}
\caption{Summary macro F1 (given abstention) and coverage metrics for the bilateral factuality evaluation function $\zeta$ and a baseline unilateral factuality evaluation function.}
\label{tab:summaryperf}
\end{table}

\paragraph{Results} Table \ref{tab:summaryperf} compares macro F1 and coverage between the unilateral and bilateral evaluations across the two datasets, grouped by whether a model's type was flagship or distilled, and Table \ref{tab:summaryperftnt} does the same for mean time of execution and mean tokens used. Table \ref{tab:summarytvd} summarizes the distribution of truth values produced in bilateral evaluation. Detailed breakouts are shown in the tables in Appendix \ref{apd:experiments}. Our key findings are as follows. 

\begin{enumerate}
    \item \textbf{\textit{Bilateral evaluation macro F1 outperforms unilateral evaluation $(p < 0.01)$ at the cost of lower coverage}}. The mean difference between bilateral and unilateral macro F1 is 0.062 and for coverage is -0.456.
    \item \textbf{\textit{Flagship models outperform distilled models $(p < 0.01)$}}. Table \ref{tab:summaryperf} shows that this is the case for both unilateral and bilateral approaches, though the difference is more pronounced with bilateral evaluation (0.091 on the GPQA dataset, 0.112 on the SimpleQA dataset) versus unilateral evaluation (0.074 on the GPQA dataset, 0.087 on the SimpleQA dataset).
    \item \textbf{\textit{Bilateral evaluation is more expensive than unilateral evaluation $(p < 0.01)$}}. Table \ref{tab:summaryperftnt} shows that bilateral evaluation takes roughly twice as much time and twice as many tokens as unilateral evaluation. However, evaluation times vary widely, with GPT-4o Mini using direct prompting taking a mean of 2.5 seconds, and Llama 4 Scout using zero-shot prompting taking a mean of 43.4 seconds. Token consumption scales predictably, from up to a mean of 2,008.6 tokens with direct prompting, and up to a mean of 6,704.7 tokens with few-shot prompting.
    \item \textbf{\textit{Inconsistency occurs significantly more frequently than incompleteness $(p < 0.05)$}}. Table \ref{tab:summarytvd} shows that bilateral judge models more frequently abstain by assigning $\langle \mathfrak{t},\mathfrak{t}\rangle$ (both verified and refuted) as opposed to  $\langle \mathfrak{f},\mathfrak{f}\rangle$ (neither verified nor refuted).
\end{enumerate}

We note that in the experiments reported here, the $\mathfrak{e}$ component of the weak Kleene truth value never occurred: every bilateral evaluation produced a pair $\langle u, v \rangle \in \{ \mathfrak{t}, \mathfrak{f} \} \times \{ \mathfrak{t}, \mathfrak{f} \}$. The otherwise clauses of Definition \ref{def:zeta} — which assign $\mathfrak{e}$ when neither the verification nor the refutation marker tokens appear in the LLM output — were therefore never triggered. We attribute this to two factors. First, the verification and refutation prompt templates (Appendix \ref{apd:templates}) instruct the LLM to terminate its response with a single line drawn from a two-element set ($\{ \mathsf{\ulcorner VERIFIED \urcorner, \ulcorner CANNOT\ VERIFY \urcorner} \}$ and $\{ \mathsf{ \ulcorner REFUTED \urcorner, \ulcorner CANNOT\ REFUTE \urcorner} \}$, respectively), which strongly constrains the output distribution. Second, all six models evaluated — three flagship (Llama 4 Maverick, GPT-4o, Claude 3.5 Sonnet) and three distilled (Llama 4 Scout, GPT-4o Mini, Claude 3.5 Haiku) — are highly reliable at following such terminal-line instructions, and we observed no API timeouts or malformed responses over the $2 \times 400 \times 18 = 14,400$ calls underlying the results in Tables \ref{tab:summaryperf}-\ref{tab:summarytvd}. The $\mathfrak{e}$ component nevertheless remains a principled and necessary part of the formalism: it absorbs genuine instruction-following failures, API errors, and timeouts without corrupting the reasoner's state, and it is likely to be exercised in longer-running deployments, under rate-limit pressure, or with smaller local models whose instruction-following reliability is lower. We retain it in Definition \ref{def:zeta} precisely for those settings.

\begin{table}[t]
\centering
\footnotesize
\begin{tabular}{llrrrrrrrr}
\toprule
& & \multicolumn{2}{c}{\textbf{Bilateral} ($\zeta$)} & \multicolumn{2}{c}{\textbf{Unilateral}} \\
\cmidrule(lr){3-4} \cmidrule(lr){5-6}
\textbf{Dataset} & \textbf{Model Type} & \textbf{Time (s)} & \textbf{Tokens} & \textbf{Time (s)} & \textbf{Tokens} \\
\midrule
GPQA & flagship & 36.747 (0.281) & 4781.663 (45.878) & 12.411 (0.140) & 2212.766 (27.182) \\
& distilled & 34.771 (0.224) & 4731.672 (41.661) & 13.439 (0.095) & 2532.153 (26.615) \\
\midrule
SimpleQA & flagship & 32.789 (0.274) & 4163.807 (37.704) & 12.256 (0.140) & 2100.465 (25.634)\\
& distilled & 30.310 (0.253) & 3964.037 (39.991) & 11.424 (0.120) & 2167.419 (28.044) \\
\bottomrule
\end{tabular}
\caption{Summary execution time (in seconds) and total tokens used for the bilateral factuality evaluation function $\zeta$ and a baseline unilateral factuality evaluation function. }
\label{tab:summaryperftnt}
\end{table}

\begin{table}[t]
\centering
\footnotesize
\begin{tabular}{llrrrrrrrr}
\toprule
\textbf{Dataset} & \textbf{Model Type} & \textbf{$\langle \mathfrak{t},\mathfrak{t} \rangle$} & \textbf{$\langle \mathfrak{t},\mathfrak{f} \rangle$} & \textbf{$\langle \mathfrak{f},\mathfrak{t} \rangle$} & $\langle \mathfrak{f},\mathfrak{f} \rangle$ \\
\midrule
GPQA & flagship & 0.301 (0.008) & 0.211 (0.007) & 0.378 (0.008) & 0.110 (0.005)  \\
& distilled & 0.299 (0.007) & 0.192 (0.006) & 0.312 (0.007) & 0.197 (0.006) \\
\midrule
SimpleQA & flagship & 0.310 (0.007) & 0.228 (0.007) & 0.357 (0.008) & 0.106 (0.005) \\
& distilled & 0.301 (0.007) & 0.233 (0.007) & 0.266 (0.007) & 0.200 (0.006) \\
\bottomrule
\end{tabular}
\caption{Summary truth value distributions for the bilateral factuality evaluation function $\zeta$. Of note is the fact that the models evaluated did not produce any truth values where $u = \mathfrak{e}$ or $v = \mathfrak{e}$ during the evaluation.}
\label{tab:summarytvd}
\end{table}

\paragraph{The accuracy-coverage trade-off}
\label{sect:accuracy-coverage}

A key finding from our experimental results is that bilateral evaluation achieves higher macro F1 than unilateral evaluation at the cost of lower coverage (Table~\ref{tab:summaryperf}). This trade-off is not a limitation but a feature: the increased visibility into the LLM's epistemic state that bilateral evaluation provides enables more conservative and reliable predictions.

The coverage reduction occurs because bilateral evaluation abstains whenever the LLM's epistemic state is indefinite---either inconsistent ($\langle \mathfrak{t}, \mathfrak{t} \rangle$) or ignorant ($\langle \mathfrak{f}, \mathfrak{f} \rangle$). Rather than forcing a commitment when evidence is conflicting or absent, the system acknowledges uncertainty. This epistemic humility is precisely what produces the accuracy improvement: by abstaining on difficult or ambiguous cases, the system's definite verdicts are more reliable.

This pattern aligns with the well-established framework of \textit{selective classification} (also known as classification with abstention or the reject option) \citep{el2010foundations,geifman2017selective}. Selective classification is motivated by applications where conservative prediction is critical---the cost of an incorrect prediction outweighs the cost of abstention. Medical diagnosis is a canonical example: a system that abstains on ambiguous cases and defers to human judgment is preferable to one that confidently misdiagnoses. The medication safety example in Section~\ref{sect:tableau} illustrates this directly: the system's abstention on \texttt{Cardiosafe(aspirin)} (a genuinely contested claim) and its glut detection on \texttt{Nonaddictive(valium)} (a claim the LLM can refute) both represent appropriate epistemic caution in a domain where errors carry significant consequences.

Bilateral evaluation provides a principled mechanism for selective classification that goes beyond simple confidence thresholding. Standard approaches to selective classification typically use model confidence scores to decide when to abstain, but such scores are notoriously miscalibrated in LLMs \citep{kadavath2022language}. Our approach instead grounds abstention in the structure of evidence: abstention occurs when verification and refutation evidence are both present (glut) or both absent (gap). This provides interpretable diagnostics---a downstream system can distinguish between ``the LLM is conflicted'' and ``the LLM lacks relevant knowledge,'' enabling different handling strategies for each case.

The connection to verbalization (Section~\ref{sect:verbalization}) is direct: gap valuations ($\langle \mathfrak{f}, \mathfrak{f} \rangle$) arise when the verbalized formula fails to activate relevant LLM knowledge. Improving verbalization strategies would shift some gaps to definite valuations, increasing coverage while maintaining accuracy. Glut valuations ($\langle \mathfrak{t}, \mathfrak{t} \rangle$), by contrast, reflect genuine LLM inconsistency that no verbalization strategy can resolve---these represent cases where the LLM's parametric knowledge is itself contradictory.

\paragraph{Bilateral truth value distributions across subject domains}

We note that our experimental benchmarks (GPQA and SimpleQA) span diverse subject matters, and the distribution of truth values likely varies across domains. Highly technical domains may produce more gaps due to verbalization challenges; domains with contested claims may produce more gluts. Table~\ref{tab:bilateral-topic-distribution} presents the distribution of bilateral truth values across ten topic categories in the SimpleQA dataset using zero-shot evaluation, revealing systematic variation in epistemic states both across models and subject domains.

Three distinct behavioral patterns emerge among model families. GPT models exhibit a pronounced tendency toward contradiction, with glut rates ($\langle \mathfrak{t},\mathfrak{t} \rangle$) averaging 59.0\% and near-zero gap rates ($\langle \mathfrak{f},\mathfrak{f} \rangle \approx 0.1\%$), indicating that these models routinely find both supporting and refuting evidence for the same claim. In contrast, Claude models display greater epistemic caution, producing substantially lower glut rates (19.6\%) and higher gap rates (30.4\%), suggesting a willingness to withhold judgment when evidence is inconclusive. Llama models occupy an intermediate position with high verification rates ($\langle \mathfrak{t},\mathfrak{f} \rangle = 29.8\%$) but also elevated glut rates (52.9\%).

Notably, topic difficulty varies considerably: Video games and Sports exhibit the highest glut rates (63.6\% and 47.3\%, respectively), while History shows the lowest (35.7\%) and the most consistent cross-model behavior ($\sigma = 15.0$). These findings demonstrate that bilateral evaluation exposes model-specific epistemic signatures that unilateral \textsc{true}/\textsc{false} judgments would collapse, and suggest that certain knowledge domains---particularly entertainment and popular culture---present greater challenges for consistent factuality assessment than more extensively documented domains such as History.

\section{A Belnap computer with paraconsistent reasoning and LLM integration}
\label{sect:tableau}
To evaluate the practical utility of our theoretical framework, we implemented a reasoning system for $\mathbf{ACrQ}$ in Python with type annotations and a modular architecture separating formula representation, semantic evaluation, and tableau construction. 

\paragraph{System architecture} The architecture supports two complementary reasoning modes:
\begin{itemize}
    \item \textbf{Forward-chaining inference} derives new atomic formulas by instantiating universally quantified rules over ground assertions. Given a theory containing rules of the form $t\!:\![\forall X\, P(X)]Q(X)$ and ground assertions $t\!:\!P(a)$, the forward-chaining procedure derives $t\!:\!Q(a)$ and iterates until no new facts can be inferred. Each newly derived atomic formula is then evaluated by the LLM judge using the caching bilateral factuality evaluation function $\zeta_c$, which may introduce additional assertions (confirming the inference) or their duals (indicating LLM refutation), potentially surfacing gluts.
    \item \textbf{Tableau-based satisfiability checking} determines whether the current theory is satisfiable under $\mathbf{ACrQ}$ semantics, extending the tableau calculus defined by \citet{ferguson2021tableaux}. The tableau procedure decomposes signed formulas according to $\mathbf{ACrQ}$ rules until branches either close (indicating unsatisfiability) or remain open (indicating satisfiability). For theories containing gluts---formulas $\varphi$ where both $t\!:\!\varphi$ and $t\!:\!\varphi^*$ hold---the paraconsistent semantics prevents logical explosion, allowing the reasoner to identify exactly which conclusions are contested while maintaining overall satisfiability.
\end{itemize}
Together, these two modes support an iterative workflow: users assert formulas in natural language, the system translates them to $\mathbf{AC}$ syntax, forward-chaining derives consequences and queries the LLM for validation, and satisfiability checking identifies any gaps (neither verified nor refuted) or gluts (both verified and refuted) in the resulting theory.

The system adds to the set of tableau rules in Definition 18 in \citep{ferguson2021tableaux} an additional LLM evaluation rule that queries external knowledge sources during tableau construction. The LLM evaluation rule is formally specified as:

\begin{prooftree}
\AxiomC{$\text{LLM}(P(a)) = \langle u, v \rangle$}
\RightLabel{(llm-eval)}
\UnaryInfC{$\sigma : P(a) \vdash \Delta_{\text{LLM}}(\sigma, P(a), u, v)$}
\end{prooftree}
\noindent where $\Delta_{\text{LLM}}$ generates conclusions based on the bilateral truth value returned by the LLM. This rule is applied with lowest priority, after all deterministic tableau rules have been exhausted. The implementation of the tableau procedure uses the caching bilateral factuality evaluation function $\zeta_c$ where valuations for atomic formulas are persistently and immutably stored, ensuring the stability of truth values throughout the reasoning process, consistent with optimization techniques used in description logic tableau reasoners \citep{gore2007exptime}. The tableau maintains full tree connectivity by chain-connecting initial formulas as siblings, ensuring all properties remain observable and verifiable throughout the reasoning process. The system preserves $\mathbf{ACrQ}$'s paraconsistent properties, handling contradictions (or ``knowledge gluts'') without logical explosion. For example, when the LLM returns $\langle\mathbf{t}, \mathbf{t}\rangle$, indicating both positive and negative evidence, the tableau construction procedure generates nodes for both $\mathbf{t}:P(a)$ and $\mathbf{t}:P^*(a)$, maintaining the glut without causing branch closure, applying Norihiro Kamide's method \citep{kamide2010paraconsistent} of encoding bilateral valuations for atomic formulas with a predicate $P$ using a predicate $P^*$ that corresponds to $P$'s \textit{refutability} value \citep{maier2013paraconsistent,ferguson2021tableaux}.

Claude Code \citep{claudecode2025} was used to accelerate implementation while maintaining correctness through rigorous verification: the system is supported by extensive test suites (311 tests covering compliance with \cite{ferguson2021tableaux}, semantic correctness, and edge cases) with all features validated through code review by two of the authors (BA and TF) before integration.\footnote{Code and data for the reasoner is available at \url{https://github.com/bradleypallen/wkrq}; in addition, the Python package used by the reasoner for the implementation of $\zeta$ and $\zeta_c$ is available at \url{https://github.com/bradleypallen/bilateral-truth}.}

To support interactive use of the reasoner, the implementation includes a dialogue management system, Theory Manager, that bridges natural language input and formal reasoning. Natural language assertions are entered by the user, translated into formulas in $\mathbf{AC}$, and added to a set of assertions stored persistently by the reasoner as a theory. The user can issue commands to perform forward-chaining inference, and to perform satisfiability checking of the current theory, including gap and glut detection. 

As an illustration of the reasoner in action, Figure \ref{fig:theory-manager-example} shows a trace of an interactive session with the Theory Manager. In this example, the user first specifies the use of OpenAI's GPT-4o model with the session, and then makes two assertions in natural language, the canonical example of all humans are mortal and Socrates is a human, which the Theory Manager translates into the $\mathbf{AC}$ syntax. The user then invokes the prover to apply rules in a forward-chaining fashion, yielding the inference that Socrates is a mortal. Given the appearance of atomic statements in the tableau being constructed, whether directly asserted or inferred, the \texttt{llm-eval} rule is being applied; given it agrees with the atomic formulas added to this point, the tableau construction procedure does not add any additional nodes. Then the user asserts that Socrates is a pig, and this time the \texttt{llm-eval} rule yields an assertion that Socrates is not a pig. Satisfiability checking invoked by the user identifies that an inconsistency now exists in the set of assertions in the theory; however, as this does not result in explosion per the definition of $\mathbf{AC}$, the theory is still satisfiable.
\begin{figure}[t]
\centering
\begin{lstlisting}[
    basicstyle=\footnotesize\ttfamily,
    frame=single,
    backgroundcolor=\color{gray!5},
    xleftmargin=0.5cm,
    xrightmargin=0.5cm,
    escapeinside={(*}{*)},
    numbers=left,
    numberstyle=\tiny\color{gray},
    stepnumber=1,
    numbersep=5pt
]
theory> llm openai gpt-4o
[OK] LLM evaluator enabled: openai

theory> assert all humans are mortal
[OK] Asserted: S0001
  Formula: t:[forall X Human(X)]Mortal(X)

theory> assert socrates is a human
[OK] Asserted: S0002
  Formula: t:Human(socrates)

theory> infer
[OK] Inferred 1 new statement(s):
  I0003: t:Mortal(socrates)              (*\textcolor{blue}{$\leftarrow$ logical inference}*)

theory> assert socrates is a pig
[OK] Asserted: S0004
  Formula: t:Pig(socrates)

theory> infer
[OK] Inferred 1 new statement(s):
  E0005: t:Pig*(socrates)                (*\textcolor{red}{$\leftarrow$ LLM refutation}*)

theory> check
[OK] Satisfiable
[!] Found 1 glut (conflicting evidence):
  S0004: t:Pig(socrates)                 (*\textcolor{orange}{$\leftarrow$ glut detected}*)
  E0005: t:Pig*(socrates)                (*\textcolor{orange}{$\leftarrow$ but satisfiable}*)
\end{lstlisting}
\caption{Example interaction with the Theory Manager demonstrating paraconsistent reasoning. The system correctly infers \texttt{Mortal(socrates)} through logical deduction (line 13), then queries the LLM when a false claim is asserted. The LLM returns $\zeta_c(\text{Pig(socrates)}) = \langle \mathfrak{f}, \mathfrak{t} \rangle$ (cannot verify, can refute), resulting in \texttt{Pig*(socrates)} (line 22). Despite the glut between the user assertion and LLM evidence (lines 24-25), the system remains satisfiable, demonstrating paraconsistent handling of contradictions.}
\label{fig:theory-manager-example}
\end{figure}

\paragraph{A medication safety reasoning use case}

To illustrate the system's capabilities for a use case where inconsistency tolerance is practically motivated, we present a medication safety example. Medical knowledge bases frequently contain errors, outdated information, or conflicting guidelines that can lead to incorrect inferences \citep{da1989paraconsistent}. Traditional approaches to managing such inconsistencies in medical ontologies focus on detecting and repairing problematic axioms through techniques such as axiom pinpointing \citep{schlobach2007debugging,kalyanpur2006debugging,baader2008debugging}. However, these methods require that inconsistencies manifest as logical contradictions within the ontology itself; they cannot detect axioms that are logically consistent but factually incorrect given external medical knowledge. Our approach addresses this gap: by validating inferred conclusions against LLM-encoded medical knowledge, we can surface errors (gluts) that would otherwise remain hidden, while the paraconsistent semantics allows reasoning to continue even when conflicts are detected. This example is deliberately artificial---it includes intentionally incorrect rules to test glut detection---but it demonstrates a capability that could complement existing ontology debugging tools in domains where knowledge bases must be validated against evolving empirical evidence.

The medication safety example knowledge base contains 228 asserted statements encoding 28 universally quantified rules and 200 ground assertions across 16 drug classes.
\begin{table}[t]
\centering
\begin{tabular}{lr@{\qquad}lr}
\toprule
Metric & Value & Glut Category & Count \\
\midrule
Asserted statements & 228 & Opioids nonaddictive & 25 \\
Inferred statements & 712 & Antipsychotics metabolically neutral & 15 \\
Satisfiability & \textsc{sat} & Beta-blockers safe in asthma & 11 \\
Gluts detected & 92 & Benzodiazepines nonaddictive & 11 \\
Gaps detected & 2 & Other categories & 30 \\
\bottomrule
\end{tabular}
\caption{Scaled evaluation results (940 statements total).}
\label{tab:scaled-results}
\end{table}
Table~\ref{tab:scaled-results} summarizes the results of asserting the statements in the knowledge base, running the forward chaining reasoner, and then running the tableau satisfiability checker to detect gluts.  Forward-chaining derived 712 new atomic formulas, yielding a theory of 940 total statements. The system detected 92 gluts corresponding to medically significant errors: all 25 opioid medications yielded gluts for $\mathsf{Nonaddictive}(X)$ because the LLM correctly identifies opioids as highly addictive. The system also correctly identified that beta-blockers are contraindicated in asthma patients (11 gluts) and that fluoroquinolones pose risks for children (7 gluts). Despite 92 contradictions, the theory remained satisfiable.
We encode rules about drug classifications, some correct and some deliberately erroneous:
\begin{align*}
\mathsf{S0001}&: \mathfrak{t}:[\forall X\, \mathsf{Benzodiazepine}(X)]\mathsf{Nonaddictive}(X) & \text{(incorrect)} \\
\mathsf{S0002}&: \mathfrak{t}:[\forall X\, \mathsf{Benzodiazepine}(X)]\mathsf{Sedative}(X) & \text{(correct)} \\
\mathsf{S0003}&: \mathfrak{t}:[\forall X\, \mathsf{Nsaid}(X)]\mathsf{Analgesic}(X) & \text{(correct)} \\
\mathsf{S0004}&: \mathfrak{t}:[\forall X\, \mathsf{Nsaid}(X)]\mathsf{Cardiosafe}(X) & \text{(incorrect)}
\end{align*}
Rule $\mathsf{S0001}$ is medically false: benzodiazepines carry significant dependence risks. Rule $\mathsf{S0004}$ is also false: NSAIDs increase cardiovascular risk. We add ground assertions about specific medications (e.g., $\mathfrak{t}:\mathsf{Benzodiazepine}(\mathsf{valium})$, $\mathfrak{t}:\mathsf{Nsaid}(\mathsf{ibuprofen})$).
Forward-chaining derives consequences such as $\mathsf{Nonaddictive}(\mathsf{valium})$ and $\mathsf{Cardiosafe}(\mathsf{ibuprofen})$. The LLM judge (GPT-4o) then evaluates each derived formula bilaterally.
\begin{table}[t]
\centering
\begin{tabular}{llll}
\toprule
Formula & KB & LLM & Status \\
\midrule
$\mathsf{Nonaddictive}(\mathsf{valium})$ & $\mathfrak{t}$ & $\langle \mathfrak{f}, \mathfrak{t} \rangle$ & Glut \\
$\mathsf{Cardiosafe}(\mathsf{ibuprofen})$ & $\mathfrak{t}$ & $\langle \mathfrak{f}, \mathfrak{t} \rangle$ & Glut \\
$\mathsf{Sedative}(\mathsf{valium})$ & $\mathfrak{t}$ & $\langle \mathfrak{t}, \mathfrak{f} \rangle$ & Agreement \\
$\mathsf{Analgesic}(\mathsf{ibuprofen})$ & $\mathfrak{t}$ & $\langle \mathfrak{t}, \mathfrak{f} \rangle$ & Agreement \\
$\mathsf{Cardiosafe}(\mathsf{aspirin})$ & $\mathfrak{t}$ & $\langle \mathfrak{f}, \mathfrak{f} \rangle$ & Gap \\
\bottomrule
\end{tabular}
\caption{LLM-grounded evaluation of derived formulas in the medication safety example. The KB column shows the truth value from symbolic inference; the LLM column shows $\zeta_c(\varphi) = \langle u, v \rangle$.}
\label{tab:medication-eval}
\end{table}

Table~\ref{tab:medication-eval} shows the results of the evaluation.
The deliberately erroneous rules produced gluts: the LLM correctly identifies that Valium carries dependence risks and that ibuprofen increases cardiovascular risk. Where rules are accurate, the LLM confirmed the derived conclusions. The gap for $\mathsf{Cardiosafe}(\mathsf{aspirin})$ reflects genuine medical ambiguity---low-dose aspirin has cardioprotective effects in some populations but bleeding risks in others.
Despite multiple gluts, the system remained satisfiable. A classical reasoner would face explosion; our paraconsistent semantics localizes contradictions, allowing identification of exactly which conclusions are contested while maintaining valid inferences from uncontested premises.

\paragraph{Computational complexity}

The forward-chaining phase computes the deductive closure of the asserted knowledge base under the rule set, producing a finite ground theory. Let $R$ denote the number of rules, $C$ the number of constants, and $k$ the maximum number of distinct variables in any rule body. Standard Datalog complexity analysis applies \citep{dantsin2001complexity}: each iteration considers $O(R \cdot C^k)$ candidate instantiations, and since each iteration must add at least one new ground atom or terminate, the fixpoint is reached in at most $O(|P| \cdot C^a)$ iterations, where $|P|$ is the number of predicates and $a$ the maximum arity. For fixed schema parameters, forward-chaining is polynomial in the number of constants, consistent with Datalog's \textsc{PTime} data complexity for ground query answering. Following the definition of the bilateral factuality evaluation function, each atom requires $n$ queries to assess support and $n$ queries to assess refutation where $n$ is the size of the repeated sample, yielding $2n$ LLM calls per inference. For $F$ inferred atoms with mean call latency $L$, total interpretation time is $O(nFL)$ sequentially, or $O(nFL/P)$ with parallelism degree $P$.

The tableau procedure checks satisfiability of the ground theory under $\mathsf{ACrQ}$ semantics via the signed tableau calculus of \citet{ferguson2021tableaux}, which provides a sound and complete proof procedure for $\mathsf{AC}$ with restricted quantifiers. Because forward-chaining eliminates quantifiers by grounding, the tableau operates over a quantifier-free theory of $N$ ground atoms. The key complexity observation, noted by \citet{ferguson2021tableaux}, is that $\mathsf{AC}$ validity is polynomial-time reducible to classical propositional validity via a translation that preserves formula structure. This translation doubles the signature (introducing $\overline{R}$ for each predicate $R$ to track refutation) but preserves formula size up to a constant factor. For ground theories, tableau satisfiability thus reduces to propositional satisfiability over $O(N)$ atoms---\textsc{NP}-complete in general \citep{cook1971complexity}, though the conjunctive structure of theories produced by forward-chaining often admits more efficient decision procedures. The tableau invokes $\mathcal{I}_{LLM}$ whenever a ground atomic formula appears as a node in the tableau tree; letting $A$ denote the number of distinct atoms encountered during expansion, semantic interpretation requires $O(nAL)$ time. Unlike classical description logics, where TBox reasoning ranges from \textsc{ExpTime}-complete for $\mathcal{ALC}$ \citep{schild1991terminological} to \textsc{2NExpTime}-complete for $\mathcal{SROIQ}$ \citep{kazakov2008riq}, $\mathsf{ACrQ}$'s grounding strategy yields complexity comparable to ABox-only reasoning. The cost of paraconsistency is not computational but rather the overhead of LLM-based semantic interpretation.

Table~\ref{tab:timing} reports timing from the medication safety use case, with $R = 28$ rules, $C = 200$ constants, $k = 2$ maximum variables per rule, $N = 940$ ground statements after closure, and $n = 3$ samples per verification/refutation direction.
\begin{table}[t]
\centering
\small
\begin{tabular}{lrrl}
\toprule
Phase & Time & LLM Calls & Notes \\
\midrule
Forward-chaining (logic) & 0.006s & --- & Grounding \\
Forward-chaining (LLM) & 463s & 4,272 & $712 \times 6$ calls at ${\sim}0.11$s \\
Tableau (logic + LLM) & 297s & varies & Expansion over 940 statements \\
\midrule
\textbf{Total} & \textbf{760s} & & \\
\bottomrule
\end{tabular}
\caption{Timing from the medication safety evaluation. LLM interpretation 
dominates computational cost.}
\label{tab:timing}
\end{table}
These results confirm that LLM latency dominates execution time: 
forward-chaining logic completes in under 10ms, while semantic interpretation accounts for over 99\% of total runtime. This cost structure indicates that moderately-sized knowledge bases are tractable with the current proof-of-concept implementation---the medication safety knowledge base, with 200 entities and 940 ground statements, completes in under 13 minutes. More significantly, the dominant costs are amenable to optimization strategies that can plausibly scale throughput by one to two orders of magnitude. First, the $2n$ LLM calls per atom are trivially parallelizable across independent atoms; with parallelism degree $P = 50$ (readily achievable with standard API rate limits), the 463s forward-chaining LLM phase would reduce to under 10s. Second, mature DL reasoners have developed sophisticated optimization techniques---including absorption, lazy unfolding, dependency-directed backjumping, and semantic branching \citep{horrocks2003implementation}---that substantially reduce tableau expansion; adapting these techniques to $\mathsf{ACrQ}$ would decrease both logical overhead and the number of LLM-interpreted atoms $A$. Third, local LLM deployment eliminates network latency entirely: current quantized models achieve sub-10ms inference on consumer hardware, which would reduce per-call latency by an order of magnitude relative to API-based inference. Combined, these optimizations suggest that knowledge bases with $10^4$--$10^5$ ground statements---comparable to medium-scale industrial ontologies---are within reach of production deployment.

\section{Discussion and limitations}
\label{sect:limitations}

We have provided a theoretical framework for integrating an LLM with a paraconsistent reasoner, demonstrated the feasibility of providing LLM-grounded valuations of atomic formulas, and demonstrated a Belnap computer that implements the framework. In this section we discuss the broader implications of these results and the limitations of our approach with respect to verbalization, computational complexity, error propagation, and belief revision.

\paragraph{Verbalization}
\label{sect:verbalization}

The bilateral factuality evaluation function $\zeta$ depends critically on the verbalization function $\delta$, which translates atomic formulas into natural language for LLM evaluation. This translation determines whether the LLM can provide a meaningful valuation or must return $\langle \mathfrak{f}, \mathfrak{f} \rangle$, indicating epistemic gap.
An atomic formula $\varphi = R(c_1, \ldots, c_n)$ is evaluable when two conditions hold: (i) the arguments $c_1, \ldots, c_n$ refer to entities about which the LLM has parametric knowledge, and (ii) the predicate $R$ expresses a property or relation the LLM can assess for those entities. When both conditions are met, the LLM can draw on its training data to provide verification or refutation evidence. When either condition fails, the LLM lacks grounds for evaluation and returns $\langle \mathfrak{f}, \mathfrak{f} \rangle$.
Consider the contrast between \texttt{Addictive(valium)} and \texttt{Addictive(compound7)}. The former is evaluable: ``Valium'' refers to a well-documented pharmaceutical (diazepam), and ``addictive'' is a property extensively discussed in medical literature. The LLM can assess this claim. The latter is not evaluable: ``compound7'' is an opaque identifier about which the LLM has no knowledge. The LLM cannot verify or refute claims about unknown entities, so it returns $\langle \mathfrak{f}, \mathfrak{f} \rangle$.
Similarly, predicate interpretability matters. The formula \texttt{Sedative(valium)} verbalizes naturally to ``Valium is a sedative,'' a claim the LLM can evaluate. But a domain-specific predicate like \texttt{Cyp3a4Inhibitor(valium)} may or may not be evaluable depending on whether the LLM's training data includes sufficient pharmacological detail about cytochrome P450 interactions. Highly technical or abbreviated predicates may fail to activate relevant LLM knowledge even when that knowledge is present.
When the LLM cannot evaluate a formula, the returned value $\langle \mathfrak{f}, \mathfrak{f} \rangle$ represents a legitimate epistemic state: the LLM has neither evidence for nor evidence against the claim. This is not a system failure. The paraconsistent semantics of $\mathbf{AC}$ handles gaps naturally---the formula's valuation remains underdetermined by the LLM, and reasoning proceeds based on symbolic inference and explicit assertions. The LLM-grounded interpretation thus acts as a \textit{partial oracle}, validating what it can and remaining silent otherwise.
This contrasts with approaches that force binary verdicts. A unilateral evaluation function that must return true or false cannot represent ignorance; it must either guess or abstain entirely. The bilateral approach using $\mathcal{NINE}$-valued semantics provides a middle path: the LLM contributes what epistemic signal it can, and the logic handles the rest.
As the medication safety example in Section~\ref{sect:tableau} illustrates, formulas with transparent predicate names referring to well-documented properties (e.g., \texttt{Analgesic(ibuprofen)}) yield definite valuations, while formulas with less evocative names (e.g., \texttt{Monitored(warfarin)}) may return gaps---not necessarily because the LLM lacks the relevant knowledge, but because the verbalization fails to activate it.
The implementation described in Section~\ref{sect:llmgrounded} uses direct syntactic mapping, providing the formula's surface syntax to the LLM (e.g., \texttt{Nonaddictive(valium)} verbalized as ``Nonaddictive(valium)''). This minimal approach relies on the LLM's ability to parse logical notation and on implicit grounding through symbol names that happen to correspond to natural language terms. While sufficient for formulas with transparent symbol names, this approach has limitations:
\begin{enumerate}
\item \textit{Opaque symbol names}: Knowledge bases often use abbreviated or systematic naming conventions (e.g., \texttt{rxn\_4521}, \texttt{hasCADRisk}) that do not transparently convey meaning.
\item \textit{Predicate-argument structure}: The relationship between predicate and arguments may require explicit verbalization. The formula \texttt{Treats(aspirin, headache)} is more naturally evaluated as ``Aspirin treats headaches'' than as ``Treats(aspirin, headache).''
\end{enumerate}
Alternative verbalization strategies can address these limitations. \textit{Template-based transformation} uses predefined patterns for each predicate: \texttt{Treats(X, Y)} $\rightarrow$ ``X treats Y''; \texttt{Contraindicated(X, Y)} $\rightarrow$ ``X is contraindicated for patients with Y''. This provides more natural language input but requires manual template creation. \textit{LLM-generated verbalization} uses a language model to produce natural language renderings of formulas, offering flexibility at the cost of introducing another potential source of error. Investigating alternative implementations of $\delta$ remains future work.

A complementary strategy for mitigating both epistemic gaps and verbalization mismatches is retrieval-augmented generation (RAG) \citep{gao2023retrieval}. When an atomic formula refers to entities or predicates that fall outside the LLM's parametric knowledge (for example, specialized axioms drawn from a proprietary pharmacovigilance database or a domain-specific ontology), the verbalization function $\delta$ can be extended to retrieve relevant supporting documents from an external corpus and include them in the prompt seen by the verification and refutation functions $P^+$ and $P^-$. This effectively supplies the LLM with evidence it would not otherwise possess, converting what would have been valuations of $\langle \mathfrak{f}, \mathfrak{f} \rangle$ into definite valuations of $\langle \mathfrak{t}, \mathfrak{f} \rangle$ or $\langle \mathfrak{f}, \mathfrak{t} \rangle$ while preserving the bilateral nature of $\zeta$. Because retrieval is implemented inside $\delta$ (and potentially, $P^+$ and $P^-$), the theoretical framework in Section \ref{sec:preservation} is unaffected, and the metalogical results of Theorems \ref{thm:soundness} and \ref{thm:completeness} apply unchanged, since they rely solely on the stability of the cached valuations (Lemma \ref{lem:stability}), not on how those valuations are produced. The integration of RAG into $\delta$, and the development of retrieval strategies that preserve the bilateral approach by treating retrieval for verification and refutation separately, is a promising direction for future work, particularly in the context of biomedical domains, where authoritative knowledge is under constant revision and highly specialized.

These considerations suggest guidelines for knowledge base design when LLM-grounded evaluation is intended:
\begin{enumerate}
\item \textit{Use transparent symbol names}: Choose constant and predicate names that correspond to natural language terms the LLM will recognize. Prefer \texttt{ibuprofen} over \texttt{drug\_0042}; prefer \texttt{CausesLiverDamage} over \texttt{hepTox}.
\item \textit{Favor well-documented entities}: Formulas involving entities with substantial presence in LLM training data (approved drugs, named diseases, public figures, established concepts) will be more reliably evaluable than those involving obscure or private entities.
\item \textit{Expect gaps for domain-specific content}: Formulas involving proprietary identifiers, internal codes, or specialized terminology will likely yield $\langle \mathfrak{f}, \mathfrak{f} \rangle$. Plan for the reasoner to handle these through symbolic inference alone.
\item \textit{Consider verbalization at design time}: When defining predicates, consider how they will verbalize. A predicate that produces awkward or ambiguous verbalizations will be harder for the LLM to evaluate.
\end{enumerate}
A natural extension would enrich the verbalization function to incorporate contextual information from the theory itself. If the knowledge base contains assertions about an otherwise-unknown entity (e.g.,  \texttt{HasCondition(patient42, RenalImpairment)}), this context could be included in the prompt when evaluating formulas involving that entity. The LLM would then evaluate not in isolation but in light of what the theory asserts about the entity. We leave exploration of such context-aware verbalization strategies to future work.

\paragraph{Scalability}
\label{sec:scalability}
The moderately-sized knowledge base used in the medication safety example raises natural questions about how ACrQ-based reasoning compares to mature description logic (DL) reasoners in terms of scalability, and what optimizations might enable further scaling.
Description logics are decidable fragments of first-order logic with well-characterized complexity. DL reasoners achieve decidability by carefully restricting expressivity, while ACrQ preserves full restricted quantification and tolerates contradictions.
\begin{table}[t]
\centering
\begin{tabular}{lll}
\toprule
\textbf{Feature} & \textbf{ACrQ} & \textbf{Description Logics} \\
\midrule
Quantification & Restricted $[\forall X\, P(X)]Q(X)$ & Guarded (role restrictions) \\
Truth values & 3 $(t, f, e)$ + 3 meta-signs & 2 (classical) \\
Contradictions & Tolerated (gluts) & Cause explosion \\
Decidability & Semi-decidable & Decidable (by design) \\
Primary use & Paraconsistent reasoning & Ontology reasoning \\
\bottomrule
\end{tabular}
\caption{Feature comparison between ACrQ and description logics.}
\label{tab:dl-comparison}
\end{table}
Mature DL reasoners (HermiT~\citep{glimm2014hermit}, Pellet~\citep{sirin2007pellet}, FaCT++~\citep{tsarkov2006factpp}) employ sophisticated optimizations developed over decades, many of which are directly transferable to our paraconsistent reasoner. These include \emph{lazy unfolding}, which unfolds definitions only when needed for clash detection (yielding 10--100$\times$ speedup for sparse theories); \emph{semantic branching}, which chooses branch order based on likelihood of closure (2--10$\times$ speedup); \emph{dependency-directed backjumping}, which skips irrelevant choice points (exponential gains in pathological cases); \emph{satisfiability caching for subformulas} (10--1000$\times$ for theories with repeated patterns); \emph{formula indexing via hash-based lookup} to reduce $O(N^2)$ to $O(N)$ average case; \emph{incremental reasoning through dependency tracking} for partial recomputation; and \emph{parallel branch exploration} for near-linear speedup with available cores~\citep{horrocks2003implementation}.
Beyond these transferable DL techniques, ACrQ's paraconsistent nature enables unique optimizations that, to our knowledge, have not been previously explored, and which we leave to future work: \emph{glut-aware pruning}, whereby known gluts are recorded and excluded from closure checking; \emph{bilateral predicate indexing}, which maintains dual index structures for $P$ and $P^*$ to enable $O(N)$ glut/gap detection; and \emph{sign-stratified processing}, which handles definite signs before branching signs.

A full cost–benefit analysis for production deployment — covering LLM API costs versus local deployment, the break-even point at which caching amortizes interpretation cost, and the sensitivity of downstream decision quality to the sample size n used in repeated sampling — is beyond the scope of the present theoretical contribution but would provide actionable guidance for practitioners. We regard such an analysis, anchored in a concrete application domain such as clinical decision support or ontology curation, as a valuable companion to the present work.

\paragraph{Error Propagation}
\label{sect:error-propagation}

A natural concern with LLM-grounded interpretations is error propagation: if the LLM judge returns an incorrect valuation for an atomic formula, does that error propagate through the reasoning system and corrupt downstream inferences? The bilateral architecture mitigates this risk in two ways.
First, errors in definite valuations ($\langle \mathfrak{t}, \mathfrak{f} \rangle$ or $\langle \mathfrak{f}, \mathfrak{t} \rangle$) behave identically to errors in classical interpretations---they produce unsound conclusions, but no worse than reasoning over any other knowledge base containing incorrect facts.
Second, and more importantly, the indefinite valuations that bilateral evaluation surfaces---gaps ($\langle \mathfrak{f}, \mathfrak{f} \rangle$) and gluts ($\langle \mathfrak{t}, \mathfrak{t} \rangle$)---act as natural error containment. When the LLM is uncertain or conflicted, this uncertainty propagates conservatively through the paraconsistent semantics rather than forcing a potentially incorrect commitment. The reasoner can identify which conclusions depend on definite versus indefinite premises, enabling downstream systems to assign different confidence levels accordingly.
For instance, in the medication safety example of Section~\ref{sect:tableau}, the gluts detected for \texttt{Nonaddictive(valium)} and \texttt{Cardiosafe(ibuprofen)} flag exactly those conclusions where the knowledge base's rules conflict with LLM knowledge, enabling a downstream system to treat these conclusions with appropriate caution.
Full error analysis, including empirical measurement of how verbalization failures and LLM miscalibration affect end-to-end reasoning accuracy, remains future work.

\paragraph{Belief Revision}
\label{sect:belief-revision}

The current implementation supports cache invalidation with replay: when external evidence indicates a cached valuation may be stale, the entry is cleared and the formula is re-evaluated on subsequent queries. This approach is orthogonal to the paraconsistent semantics: where paraconsistency tolerates contradictions within a knowledge state, belief revision addresses changes between knowledge states over time.
Recent work on paraconsistent belief revision \citep{testa2017agm,coniglio2024logic} has adapted AGM-style postulates \citep{alchourron1985logic} to logics of formal inconsistency, but retains the AGM motivation of restoring consistency when new information contradicts prior beliefs. Our framework takes a different stance: gluts ($\langle \mathfrak{t}, \mathfrak{t} \rangle$) are valid epistemic states reflecting genuine LLM inconsistency, not errors requiring repair. Whether and how to integrate revision mechanisms that respect this tolerance, or to explore alternatives such as truth maintenance systems \citep{forbus1993building}, remains future work.

\section{Conclusion and future work}
\label{sect:conclusion}

We described a novel approach to logical reasoning using LLMs with several key contributions. We defined a bilateral approach to factuality evaluation that identifies gaps and contradictions in LLM parametric knowledge. We introduced the concept of LLM-grounded interpretations that integrate an LLM directly into the formal semantics of the underlying logic while preserving its soundness and completeness. We provided empirical evidence that bilateral factuality evaluation outperforms unilateral approaches, and demonstrated through concrete examples how the system handles contradictory information without logical explosion. And finally, we implemented a Belnap computer that demonstrates the practical feasibility of this approach, including an interactive dialogue management interface that bridges natural language and formal reasoning. Our evaluation with an example medication safety knowledge base of 940 statements (228 asserted, 712 inferred) demonstrates that LLM-grounded reasoning maintains its theoretical properties at realistic knowledge base sizes. The system successfully identified 92 gluts where deductively inferred conclusions conflicted with LLM medical knowledge---including medically significant errors such as claiming opioids are non-addictive, beta-blockers are safe in asthma, and fluoroquinolones are safe for children. Despite these contradictions, the paraconsistent semantics preserved satisfiability and allowed identification of exactly which rules conflicted with empirical medical knowledge. Computational analysis reveals that the primary bottleneck is LLM API latency rather than logical reasoning, suggesting that standard optimization techniques from description logic reasoners---including batching, parallelization, and caching---could enable scaling to knowledge bases of over one hundred thousand axioms.

In addition to the future work described in Section \ref{sect:limitations}, we plan to 
conduct a user evaluation of the usability of the Theory Manager for ontology and knowledge graph construction, and to extend our approach based on work in the area of generalized truth values \citep{shramko2005some,hornischer2025iterating} to provide a multi-valued semantics for modeling factuality in LLMs, with the goal of improving the theoretical framework and metrics for factuality evaluation. A further valuable direction for future work would be an empirical comparison of paraconsistent and classical reasoning over LLM-grounded knowledge. Such a study could quantify the practical value of tolerating gluts rather than treating them as inconsistencies requiring repair, by measuring (i) how often classical reasoners over the same LLM-grounded atoms are forced into trivialization or require ad hoc repair, and (ii) the downstream utility of the glut and gap diagnostics surfaced by the paraconsistent reasoner for tasks such as ontology debugging, knowledge-base curation, and safety-critical decision support. Finally, a complementary application-focused study reporting a detailed cost–benefit analysis in a production setting (including monetary cost per inference, latency budgets, caching behavior, and the operational value of the glut and gap diagnostics) would translate the engineering directions sketched in Section \ref{sect:limitations} into concrete guidance for real-world deployments.

\begin{acks}
This work was partially supported by the EU’s Horizon Europe research and innovation programme within the ENEXA project (grant Agreement no. 101070305). Particular thanks are due to Frank van Harmelen for his valuable feedback based on a close reading of an earlier version of this article, and to Fabian Hoppe, Levin Hornischer, Jan-Christoph Kalo, and Lise Stork for discussions and perceptive observations that enriched our research.
\end{acks}

\bibliographystyle{SageH}
\bibliography{references}

\appendix

\section{Angell's logic of analytic containment \texorpdfstring{$\mathbf{AC}$}{AC}}
\label{apd:ac}

Below we summarize the definitions of the object language, truth functions, and interpretations for the version of $\mathbf{AC}$ presented in greater detail in \cite{ferguson2021tableaux}.

\subsection{Object language}

\begin{definition}
\label{def:objlang}    
Let $\mathcal{L}$ be a first-order language built from a countable set $\mathcal{C}$ of constants, a countable set of variables $\mathcal{V}$, a countable set $\mathcal{R}$ of relation symbols, the Boolean connectives $\neg$, $\land$, and $\lor$, restricted universal and existential quantifiers $\forall$ and $\exists$, and round parentheses (as used in complex formulas) and square brackets (as used in quantified formulas) as auxiliary symbols. If $R \in \mathcal{R}$ and $c_1$, \ldots, $c_n \in \mathcal{C}$, then $R(c_1, \ldots, c_n)$ is an atomic formula. Let $\mathcal{L}_{AT}$ be the set of atomic formulas. The formulas of $\mathcal{L}$ are the elements of $\mathcal{L}_{AT}$, together with the following, where $\varphi$, $\psi \in \mathcal{L}$ and $x \in \mathcal{V}$:
\begin{align*}
    \neg \varphi \: \vert \: (\varphi \land \psi) \: \vert \: (\varphi \lor \psi) \: \vert \: [ \forall x \varphi(x) ] \psi(x) \: \vert \: [ \exists x \varphi(x) ] \psi(x)
\end{align*}
\end{definition}

\subsection{Truth functions}

\begin{definition}
\label{def:wktruthtables}
The \emph{weak Kleene truth tables} over the
set of truth values $\mathcal{V}_3 = \{\mathfrak{t}, \mathfrak{e}, \mathfrak{f}\}$ are:
\begin{center}
\begin{tabular}{c|c}
$\neg$ & \\
\hline
$\mathfrak{t}$ & $\mathfrak{f}$ \\
$\mathfrak{e}$ & $\mathfrak{e}$ \\
$\mathfrak{f}$ & $\mathfrak{t}$
\end{tabular}
\qquad
\begin{tabular}{c|ccc}
$\land$ & $\mathfrak{t}$ & $\mathfrak{e}$ & $\mathfrak{f}$ \\
\hline
$\mathfrak{t}$ & $\mathfrak{t}$ & $\mathfrak{e}$ & $\mathfrak{f}$ \\
$\mathfrak{e}$ & $\mathfrak{e}$ & $\mathfrak{e}$ & $\mathfrak{e}$ \\
$\mathfrak{f}$ & $\mathfrak{f}$ & $\mathfrak{e}$ & $\mathfrak{f}$
\end{tabular}
\qquad
\begin{tabular}{c|ccc}
$\lor$ & $\mathfrak{t}$ & $\mathfrak{e}$ & $\mathfrak{f}$ \\
\hline
$\mathfrak{t}$ & $\mathfrak{t}$ & $\mathfrak{e}$ & $\mathfrak{t}$ \\
$\mathfrak{e}$ & $\mathfrak{e}$ & $\mathfrak{e}$ & $\mathfrak{e}$ \\
$\mathfrak{f}$ & $\mathfrak{t}$ & $\mathfrak{e}$ & $\mathfrak{f}$
\end{tabular}
\end{center}
\end{definition}

The weak Kleene truth tables for conjunction and disjunction induce the truth functions $\dot{\land}$ and $\dot{\lor}$, respectively. 

\begin{definition}
\label{def:quantifiers}
The \emph{restricted Kleene quantifier functions} $\dot{\forall}$ and $\dot{\exists}$ are mappings from sets of truth values to truth values such that:
\[
\begin{aligned}
\dot{\exists}(X) &= \begin{cases}
\mathfrak{t} & \text{if } \langle \mathfrak{t}, \mathfrak{t} \rangle \in X \\
\mathfrak{e} & \text{if for all } \langle u \text{, } v \rangle \text{, either } u = \mathfrak{e} \text{ or } v = \mathfrak{e} \\
\mathfrak{f} & \text{if } \langle \mathfrak{t}, \mathfrak{t} \rangle \notin X \text{ and for some } \langle u, v \rangle \in X \text{, } u \neq \mathfrak{e} \text{ and } v \neq \mathfrak{e}
\end{cases} \\
\dot{\forall}(X) &= \begin{cases}
\mathfrak{t} & \text{if } \langle \mathfrak{t}, \mathfrak{f} \rangle \text{, } \langle \mathfrak{t}, \mathfrak{e} \rangle \notin X \text{ and for some } \langle u \text{, } v \rangle \in X, u \neq \mathfrak{e} \text{ and } v \neq \mathfrak{e} \\
\mathfrak{e} & \text{if for all } \langle u \text{, } v \rangle \in X \text{, either } u = \mathfrak{e} \text{ or } v = \mathfrak{e} \\
\mathfrak{f} & \text{if } \{ \langle \mathfrak{t}, \mathfrak{t} \rangle \text{, } \langle \mathfrak{t} \text{, } \mathfrak{e} \rangle \} \cap X \neq \emptyset \text{ and for some } \langle u, v \rangle \in X \text{, either } u = \mathfrak{e} \text{ or } v = \mathfrak{e}
\end{cases}
\end{aligned}
\] 
\end{definition}

\subsection{Interpretations}
\label{sect:interpretations}

\begin{definition}
\label{def:standard}
    An \emph{$\mathbf{AC}$ interpretation} $\mathcal{I}$ is a pair $\langle \mathbf{C}^\mathcal{I}, \mathbf{R}^\mathcal{I} \rangle$ where $\mathbf{C}^\mathcal{I}$ is a domain of individuals and $\mathbf{R}^\mathcal{I}$ is a set of functions where $\mathcal{I}$ assigns:
    \begin{itemize}
        \item every constant $c \in \mathbf{C}$ an individual $c^\mathcal{I} \in \mathbf{C}^\mathcal{I}$
        \item every n-ary predicate $R$ a function $R^\mathcal{I}: (\mathbf{C}^\mathcal{I})^n \rightarrow \mathcal{V}_3 \times \mathcal{V}_3$
    \end{itemize}
\end{definition}

\begin{definition}
\label{def:stdmap}
An $\mathbf{AC}$ interpretation $\mathcal{I}$ induces a map $\mathcal{I}: \mathcal{L} \rightarrow \mathcal{V}_3 \times \mathcal{V}_3$ as follows, where $\mathcal{I}_0$ and $\mathcal{I}_1$ project the first and second coordinates respectively:
\begin{itemize}
\item For atomic formulas $R(c_1, \ldots, c_n) \in \mathcal{L}_{AT}$, $\mathcal{I}(\varphi) = R^\mathcal{I}(c_1^\mathcal{I}, \ldots, c_n^\mathcal{I})$
\item $\mathcal{I}(\neg \varphi) = \langle \mathcal{I}_1(\varphi), \mathcal{I}_0(\varphi) \rangle$
\item $\mathcal{I}(\varphi \land \psi) = \langle \mathcal{I}_0(\varphi) \, \dot{\land} \, \mathcal{I}_0(\psi), \mathcal{I}_1(\varphi) \, \dot{\lor} \, \mathcal{I}_1(\psi) \rangle$
\item $\mathcal{I}(\varphi \lor \psi) = \langle \mathcal{I}_0(\varphi) \, \dot{\lor} \, \mathcal{I}_0(\psi), \mathcal{I}_1(\varphi) \, \dot{\land} \, \mathcal{I}_1(\psi) \rangle$
\item $\mathcal{I}([\forall x\varphi(x)]\psi(x)) = \langle \dot{\forall}(\{\mathcal{I}_0(\varphi(c)), \mathcal{I}_0(\psi(c)) \mid c \in \mathcal{C}\}), \dot{\exists}(\{\mathcal{I}_0(\varphi(c)), \mathcal{I}_1(\psi(c)) \mid c \in \mathcal{C}\}) \rangle$
\item $\mathcal{I}([\exists x\varphi(x)]\psi(x)) = \langle \dot{\exists}(\{\mathcal{I}_0(\varphi(c)), \mathcal{I}_0(\psi(c)) \mid c \in \mathcal{C}\}),  \dot{\forall}(\{\mathcal{I}_0(\varphi(c)), \mathcal{I}_1(\psi(c)) \mid c \in \mathcal{C}\}) \rangle$
\end{itemize}
\end{definition}

\begin{definition}
\label{def:validity}
    Given an $\mathbf{AC}$ interpretation $\mathcal{I}$, validity with respect to $\mathcal{I}$ is defined as truth preservation, i.e.
    \begin{align*}
        \Gamma \models_{\mathcal{I}} \varphi \text{ if for all instances of } \mathcal{I} \text{ such that } \forall \psi \in \Gamma \: \mathcal{I}_0(\psi) = \mathfrak{t} \text{, } \mathcal{I}_0(\varphi) = \mathfrak{t}.
    \end{align*}
\end{definition}

\section{Prompts}
\label{apd:templates}

\subsection{Direct verification template}
\label{apd:dirvertemplate}
\begin{lstlisting}[basicstyle=\tiny\ttfamily, frame=single, breaklines=true]
Determine whether the following answer to the given question is correct. 
Conclude with a single line containing ONLY one of these two phrases:
VERIFIED
CANNOT VERIFY

Question: {question}
Proposed answer: {answer}
\end{lstlisting}

\subsection{Direct refutation template}
\label{apd:dirreftemplate}
\begin{lstlisting}[basicstyle=\tiny\ttfamily, frame=single, breaklines=true]
Determine whether the following answer to the given question can be refuted. 
Conclude with a single line containing ONLY one of these two phrases:
REFUTED
CANNOT REFUTE

Question: {question}
Proposed answer: {answer}
\end{lstlisting}

\subsection{Zero-shot verification template}
\label{apd:zerovertemplate}
\begin{lstlisting}[basicstyle=\tiny\ttfamily, frame=single, breaklines=true]
I'll provide you with a question and its proposed answer. 
Your task is to verify whether this answer is correct by following these steps:

1. Analyze the exact meaning of both the question and answer, 
identifying any key terms that need clarification.
2. Establish specific conditions that would make this answer true for this question.
3. Provide direct evidence supporting the answer, including specific facts, examples, 
or authoritative references that confirm its accuracy.
4. Test if the answer remains valid across all contexts where the question applies, 
noting any limitations or exceptions.
5. Check for consistency with established knowledge in the relevant domain.

Based on your analysis, determine whether the answer is verified and explain
your reasoning with specific supporting evidence. 
Your goal is not to find fault but to determine if positive 
evidence exists to confirm the answer.

After your complete analysis, conclude with a single line containing 
ONLY one of these two phrases:
VERIFIED
CANNOT VERIFY

Question: {question}
Proposed answer: {answer}
\end{lstlisting}

\subsection{Zero-shot refutation template}
\label{apd:zeroreftemplate}
\begin{lstlisting}[basicstyle=\tiny\ttfamily, frame=single, breaklines=true]
I'll provide you with a question and its proposed answer. 
Your task is to determine if this answer can be refuted by following these steps:

1. Analyze the exact meaning of both the question and the proposed answer.
2. Identify what specific conditions would need to be true for this answer to be false 
(not merely the absence of evidence).
3. Search for direct counterexamples or contradicting evidence that 
actively demonstrates why the answer is incorrect.
4. Construct specific scenarios where the answer fails to hold true, 
even if the question's premises are accepted.
5. Identify any logical inconsistencies, factual errors, or category mistakes
within the answer.

Focus on building an affirmative case for why the answer is incorrect, 
rather than simply noting a lack of supporting evidence. 
Provide specific counterevidence and explain precisely 
how it contradicts the proposed answer.

After your complete analysis, conclude with a single line containing 
ONLY one of these two phrases:
REFUTED
CANNOT REFUTE

Question: {question}
Proposed answer: {answer}
\end{lstlisting}

\subsection{Few-shot verification template}
\label{apd:fewvertemplate}
\begin{lstlisting}[basicstyle=\tiny\ttfamily, frame=single, breaklines=true]
I'll provide you with a question and its proposed answer. 
Your task is to verify whether this answer is correct by following these steps:

1. Analyze the exact meaning of both the question and answer, identifying any key terms that need clarification.
2. Establish specific conditions that would make this answer true for this question.
3. Provide direct evidence supporting the answer, including specific facts, examples, or authoritative references that confirm 
its accuracy.
4. Test if the answer remains valid across all contexts where the question applies, noting any limitations or exceptions.
5. Check for consistency with established knowledge in the relevant domain.

Based on your analysis, determine whether the answer is verified and explain your reasoning with specific supporting evidence. 
Your goal is not to find fault but to determine if positive evidence exists to confirm the answer.

Here are some examples of how to approach verification:

Example 1:
Question: Who was the first man to walk on the Moon?
Proposed answer: Neil Armstrong
Let's verify this answer systematically:
Key Terms Analysis:
"First man to walk on the Moon" requires precise historical verification
Specific date and mission details are crucial
Conditions for Verification:
Must be the first human to physically step onto lunar surface
Must have documented evidence of this historic moment
Must be part of a specific NASA mission
Direct Evidence:
Neil Armstrong stepped onto the Moon's surface on July 20, 1969
During Apollo 11 mission
Famous quote: "That's one small step for man, one giant leap for mankind"
Extensive NASA documentation confirms this
Broadcast live on global television
Verified by multiple independent sources including NASA, scientific records, and historical archives
Context and Limitations:
Specifically refers to first human, not first spacecraft
Specifically lunar surface, not orbital passage
Confirmed during Cold War space race
Part of NASA's Apollo program
Consistency Check:
Matches scientific and historical consensus
No credible alternative claims exist
Extensively documented in multiple international records
This answer meets all verification criteria with robust, multi-source evidence.
VERIFIED

Example 2:
Question: Chocolate is made from material
Proposed answer: Sugar
Let's analyze this systematically:
Key Term Analysis:
"Chocolate" refers to a food product derived from cacao/cocoa beans
"Material" suggests the fundamental ingredient or component
"Sugar" is a sweetener, not the primary material of chocolate
Conditions for Accuracy:
Sugar is an ingredient in chocolate
But sugar is NOT the primary material of chocolate
Chocolate's primary material is cocoa/cacao solids
Direct Evidence:
Chocolate is primarily made from cocoa beans processed into cocoa solids
Cocoa solids come from cacao tree seeds/beans
Sugar is added as a sweetener, not the base material
Chocolate composition typically includes:

Cocoa solids (primary material)
Cocoa butter
Sugar (secondary ingredient)
Milk (in milk chocolate)

Context Testing:
In all chocolate production processes, cocoa is the fundamental material
Sugar is always a supplementary ingredient, not the base material
Domain Consistency:
Culinary and food science consistently define cocoa/cacao as chocolate's primary material
Based on comprehensive analysis, the proposed answer is incorrect.
CANNOT VERIFY

After your complete analysis, conclude with a single line containing ONLY one of these two phrases:
VERIFIED
CANNOT VERIFY

Question: {question}
Proposed answer: {answer}
\end{lstlisting}

\subsection{Few-shot refutation template}
\label{apd:fewreftemplate}
\begin{lstlisting}[basicstyle=\tiny\ttfamily, frame=single, breaklines=true]
I'll provide you with a question and its proposed answer. 
Your task is to determine if this answer can be refuted by following these steps:

1. Analyze the exact meaning of both the question and the proposed answer.
2. Identify what specific conditions would need to be true for this answer to be false (not merely the absence of evidence).
3. Search for direct counterexamples or contradicting evidence that actively demonstrates why the answer is incorrect.
4. Construct specific scenarios where the answer fails to hold true, even if the question's premises are accepted.
5. Identify any logical inconsistencies, factual errors, or category mistakes within the answer.

Focus on building an affirmative case for why the answer is incorrect, rather than simply noting a lack of supporting evidence. 
Provide specific counterevidence and explain precisely how it contradicts the proposed answer.

Here are some examples of how to approach refutation:

Example 1:
Question: Are penguins birds?
Proposed answer: No
Let's analyze this systematically:
Meaning Analysis:
Question asks about the taxonomic classification of penguins
Proposed answer claims penguins are NOT birds
Conditions for Falsity:
Penguins must meet standard biological criteria for birds
Must share key avian characteristics
Counterevidence:
Penguins have ALL classic bird characteristics:

Feathered body
Lay eggs
Warm-blooded
Have beaks
Descended from dinosaur lineage
Classified in scientific taxonomy under Class Aves
Specifically, penguins belong to the order Sphenisciformes, which is a recognized bird order
Biological and genetic evidence conclusively places penguins within bird classification

Specific Scenarios Contradicting Answer:
Penguins have wing-like flippers adapted for swimming
They have respiratory and skeletal structures identical to other bird species
Genetic sequencing confirms their bird lineage
Logical Inconsistencies:
Rejecting penguins as birds would require rejecting fundamental biological classification systems
No scientific basis exists for excluding penguins from bird category
REFUTED

Example 2:
Question: Who was the first man to walk on the Moon?
Proposed answer: Neil Armstrong
Let's analyze this systematically:
Meaning Analysis:
Question seeks the definitive first human male to set foot on lunar surface
Proposed answer: Neil Armstrong (Apollo 11 mission, July 20, 1969)
Potential Conditions for Falsity:
Documented evidence of another person walking on Moon before Armstrong
Proof that Armstrong was not actually the first
Historical record showing a different individual preceded him
Counterevidence Search:
No credible historical evidence exists contradicting Armstrong's first Moon walk
NASA records and global documentation consistently confirm Armstrong as first
Extensive photographic and video evidence supports this claim
Scenario Testing:
No alternative scenarios emerge that could plausibly replace Armstrong's achievement
Extensive verification by multiple nations and independent researchers confirms his primacy
Logical Consistency Check:
Armstrong's Moon walk is extensively documented
Multiple witnesses and technological records corroborate the event
No logical inconsistencies detected in the claim
The proposed answer is completely accurate and supported by overwhelming historical evidence.
CANNOT REFUTE

After your complete analysis, conclude with a single line containing ONLY one of these two phrases:
REFUTED
CANNOT REFUTE

Question: {question}
Proposed answer: {answer}
\end{lstlisting}

\subsection{Prompt template for generating negative examples for SimpleQA-derived benchmark}
\label{apd:synnegtemplate}
\begin{lstlisting}[basicstyle=\tiny\ttfamily, frame=single, breaklines=true]
You are an expert synthetic data generator. Your task is to generate three plausible but 
incorrect answers to a given question that will serve as challenging distractors.

Guidelines for generating high-quality wrong answers:
1. Each answer must be factually incorrect but highly plausible within the context
   - Draw from the same domain/topic as the correct answer
   - Use answers that could reasonably be mistaken for the truth
   - Avoid obviously wrong or nonsensical options

2. Strictly match the answer type and format
   - For dates: Use the same date format and plausible timeframe
   - For people: Match profession, era, and relevance 
   - For numbers: Stay within reasonable orders of magnitude
   - For places: Use locations of similar type/scale
   
3. Ensure clear differentiation
   - Make each wrong answer distinct from the correct answer
   - Avoid overlap between wrong answers
   - Space out numerical answers appropriately
   
4. Maintain consistent specificity
   - Match the level of detail in the correct answer
   - If the answer is detailed, provide equally detailed wrong options
   - If the answer is brief, keep wrong answers similarly concise

5. Optimize for realism and difficulty
   - Wrong answers should feel natural, not contrived
   - Target common misconceptions when possible
   - Make the options challenging enough to test real knowledge
   - Avoid answers that are too obviously incorrect

Example 1:
Question: Who was the first president of the United States?
Answer: George Washington
Wrong Answers:
- John Adams
- Thomas Jefferson
- Benjamin Franklin
Reason: All are founding fathers but not the first president

Example 2:
Question: What is the largest planet in our solar system?
Answer: Jupiter
Wrong Answers:
- Saturn
- Neptune
- Uranus
Reason: All are gas giant planets, but smaller than Jupiter

Example 3:
Question: Who wrote Romeo and Juliet?
Answer: William Shakespeare
Wrong Answers:
- Christopher Marlowe
- Ben Jonson
- John Webster
Reason: All are prominent Elizabethan playwrights

Return only three wrong answers as a list in JSON format with the following requirements:
- Each wrong answer should be a string
- The output should be a single JSON object with key "negative_answers" 
- The value should be an array of exactly 3 wrong answers
- No explanations or additional text should be included
- The answers should maintain consistent formatting with the correct answer

Example format:
{{
    "negative_answers": ["Wrong Answer 1", "Wrong Answer 2", "Wrong Answer 3"]
}}
\end{lstlisting}

\subsection{Direct prompt template for unilateral factuality evaluation}
\label{apd:unidirtemplate}
\begin{lstlisting}[basicstyle=\tiny\ttfamily, frame=single, breaklines=true]
Determine whether the following answer to the given question is TRUE or FALSE.

Conclude with a single line containing ONLY one of these two words:
TRUE
FALSE

Question: {question}
Proposed answer: {answer}
\end{lstlisting}

\subsection{Zero-shot prompt template for unilateral factuality evaluation}
\label{apd:unizerotemplate}
\begin{lstlisting}[basicstyle=\tiny\ttfamily, frame=single, breaklines=true]
I'll provide you with a question and its proposed answer.
Your task is to evaluate whether this answer is correct by following these steps:

1. Analyze the exact meaning of both the question and answer, identifying any key terms that need clarification.
2. Establish specific conditions that would make this answer true for this question.
3. Provide direct evidence supporting the answer, including specific facts, examples, or 
authoritative references that confirm its accuracy.
4. Test if the answer remains valid across all contexts where the question applies, noting any limitations or exceptions.
5. Check for consistency with established knowledge in the relevant domain.
6. Search for direct counterexamples or contradicting evidence that actively demonstrates why the answer is incorrect.
7. Construct specific scenarios where the answer fails to hold true, even if the question's premises are accepted.
8. Identify any logical inconsistencies, factual errors, or category mistakes within the answer.

After analyzing the question and answer, provide a single line containing ONLY one of these two words:
TRUE
FALSE

Question: {question}
Proposed answer: {answer}
\end{lstlisting}

\subsection{Few-shot prompt template for unilateral factuality evaluation}
\label{apd:unifewtemplate}
\begin{lstlisting}[basicstyle=\tiny\ttfamily, frame=single, breaklines=true]
I'll provide you with a question and its proposed answer.
Your task is to evaluate whether this answer is correct by following these steps:

1. Analyze the exact meaning of both the question and answer, identifying any key terms that need clarification.
2. Establish specific conditions that would make this answer true for this question.
3. Provide direct evidence supporting the answer, including specific facts, examples, or 
authoritative references that confirm its accuracy.
4. Test if the answer remains valid across all contexts where the question applies, noting any limitations or exceptions.
5. Check for consistency with established knowledge in the relevant domain.
6. Search for direct counterexamples or contradicting evidence that actively demonstrates why the answer is incorrect.
7. Construct specific scenarios where the answer fails to hold true, even if the question's premises are accepted.
8. Identify any logical inconsistencies, factual errors, or category mistakes within the answer.

Here are examples of how to approach evaluation:

Example 1:
Question: Who was the first man to walk on the Moon?
Proposed answer: Neil Armstrong
Analyze the question and answer:
Question: "Who was the first man to walk on the Moon?" This is a straightforward factual question seeking the identity of 
the first human to set foot on the lunar surface.
Proposed answer: "Neil Armstrong" This is a name, presumably offered as the answer to the question.
Establish conditions for truth:
The answer is true if Neil Armstrong was indeed the first human to walk on the Moon.
Provide supporting evidence:
Historical records, NASA documentation, and countless reliable sources confirm that Neil Armstrong was the first person to 
walk on the Moon on July 20, 1969, during the Apollo 11 mission.
Test validity across contexts:
The answer holds true in all historical contexts related to the first Moon landing.
Check for consistency with established knowledge:
The answer aligns perfectly with established historical and scientific knowledge.
Search for counterexamples:
There are no credible counterexamples. No other individual is historically recognized as the first person to walk on the Moon.
Construct failure scenarios:
There are no scenarios where the answer fails, assuming the question refers to the generally accepted historical event.
Identify logical inconsistencies:
There are no logical inconsistencies or factual errors.
TRUE

Example 2:
Question: What is the main ingredient in chocolate?
Proposed answer: Sugar
Analyze the question and answer:
Question: "Chocolate is made from material" - This is an incomplete sentence. The question is implicitly asking "What material is 
chocolate made from?" or "What is a key material used to make chocolate?".
Proposed answer: "Sugar" - This suggests that sugar is the material chocolate is made from.
Establish conditions for truth:
The answer would be true if sugar was the only ingredient in chocolate, or 
if the question was interpreted as "Is sugar a material used to make chocolate?".
Provide supporting evidence:
Sugar is a common and significant ingredient in most chocolate recipes.
Test validity across contexts:
This answer fails in many contexts. Chocolate is not only made from sugar.
Check for consistency with established knowledge:
Chocolate is made from cacao beans, sugar, and often other ingredients like milk solids, cocoa butter, lecithin, and flavorings.
Search for counterexamples:
Dark chocolate often contains a higher percentage of cacao and less sugar.
Sugar-free chocolate exists, using artificial sweeteners instead.
Cacao beans are essential for chocolate, and chocolate cannot be made without them.
Construct failure scenarios:
Imagine a recipe for 100% cacao chocolate. It would contain no sugar.
Imagine a sugar-free chocolate bar. It would contain no sugar.
Identify logical inconsistencies:
The answer implies sugar is the only ingredient, which is false.
FALSE

Question: {question}
Proposed answer: {answer}
\end{lstlisting}

\section{Example}
\label{apd:examples}

To illustrate bilateral evaluation, we present an example with statements about penguins in the context of a knowledge base with a universally quantified statement that all birds can fly:

\begin{enumerate}
\item $\mathbf{C} = \{ penguin, eagle, sparrow, ... \}$
\item $\varphi_0 = [\: \forall x \: bird(x) \:] \: flies(x)$
\item $\varphi_1 = \text{bird}(\text{penguin})$
\item $\varphi_2 = \text{flies}(\text{penguin})$
\item $\varphi_3 = \neg \varphi_2$
\item $\delta(\varphi_1) = \ulcorner \text{Penguins are birds} \urcorner$
\item $P^+(\delta(\varphi_1)) = \ulcorner \ldots \text{Penguins are scientifically classified as birds. They belong to }$ \\ 
$ \text{the family Spheniscidae} \ldots \text{Conclusion: } \mathsf{VERIFIED} \urcorner$
\item $P^-(\delta(\varphi_1)) = \ulcorner \ldots \text{All evolutionary biologists classify penguins as birds. This is }$ \\ $ \text{supported by molecular evidence, fossil records, and anatomical features. }$ \\ $\text{There is no reasonable alternative classification. Conclusion: } \mathsf{CANNOT\ REFUTE} \urcorner$
\item $\zeta(\varphi_1) = \langle t, f \rangle$
\item $\delta(\varphi_2) = \ulcorner \text{Penguins fly} \urcorner$
\item $P^+(\delta(\varphi_2)) = \ulcorner \ldots \text{While penguins have wings, they cannot achieve aerial flight. }$ \\ $ \text{Their wings are adapted for swimming rather than flying. They flap their } $ \\ $ \text{wings underwater to ``fly" through water. From a strict biological perspective, } $ \\ $ \text{penguins do not fly through air. Conclusion: } \mathsf{CANNOT\ VERIFY} \urcorner$
\item $P^-(\delta(\varphi_2)) = \ulcorner \ldots \text{Penguins are flightless birds. Their wings have evolved into }$ \\ $ \text{flippers for aquatic propulsion rather than aerial flight. This is well-established }$ \\ $ \text{in ornithology. Conclusion: } \mathsf{REFUTED} \urcorner$
\item $\zeta(\varphi_2) = \langle f, t \rangle$
\item $\mathcal{I}(\varphi_0) = \langle\dot{\forall}(\{\mathcal{I}_0(\text{bird}(c)), \mathcal{I}_0(\text{flies}(c)) \mid c \in \mathbf{C}\}), \dot{\exists}(\{\mathcal{I}_0(\text{bird}(c)), \mathcal{I}_1(\text{flies}(c)) \mid c \in \mathbf{C}\})\rangle \\ = \langle f, t \rangle$
\item $\mathcal{I}(\varphi_3) = \langle \mathcal{I}_1(\varphi_2), \mathcal{I}_0(\varphi_2) \rangle = \langle t, f \rangle$
\end{enumerate}
This bilateral evaluation reveals the inconsistency. The universal statement $\varphi_0$ evaluates to false when considering penguins, and $\varphi_3$ evaluates to true, but both statements can coexist in $\mathbf{AC}$ without causing explosion. This demonstrates how the system handles the classic penguin problem through paraconsistent reasoning.

\section{Experiments}
\label{apd:experiments}

\subsection{Datasets}

We used the short-form factuality benchmarks GPQA \citep{rein2023gpqa} and SimpleQA \citep{wei2024measuring} to create the benchmarks for our experiments. GPQA consists of 448 multiple-choice questions, written by domain experts in biology, physics, and chemistry. SimpleQA consists of 4,326 question/answer pairs addressing a range of general topic areas, including history, science and technology, art, geography, TV shows, and video games. From these two benchmarks we created a balanced set of positive and negative examples. From SimpleQA, we sampled without replacement 200 question/answer pairs to be positive examples, and 200 questions to be negative examples, where we substituted false answers synthetically generated using GPT-4o Mini using the prompt shown in Appendix \ref{apd:synnegtemplate}. From GPQA, we sampled 200 existing question/answer pairs as positive examples, and 200 questions paired with the first incorrect answer for that question provided as part of the dataset.

\subsection{Experimental setup}

Following \cite{wei2024measuring}, we evaluated our implementation of $\zeta$ on a selective classification with a binary abstention task \citep{el2010foundations} using the above two datasets, measuring an LLM judge's grading of a given question/answer pair. The standard pattern in LLM judges in factuality evaluation is to prompt the LLM judge to evaluate the answer to the question as either correct, incorrect, or not attempted. This, again, has a natural mapping to the values of $\mathcal{V}_3$; to derive a single truth value $v \in \mathcal{V}_3$ for the evaluation, we use the following projection $p: \mathcal{V}_3 \times \mathcal{V}_3 \rightarrow \mathcal{V}_3$ such that for a pair $\langle u,v \rangle$: 

\begin{align*}
    p(\langle u,v \rangle) = \begin{cases}
    \mathfrak{t} & \text{if } \langle u,v\rangle = \langle \mathfrak{t},\mathfrak{f}\rangle\\
    \mathfrak{f} & \text{if } \langle u,v\rangle = \langle \mathfrak{f},\mathfrak{t}\rangle\\
    \mathfrak{e} & \text{otherwise}
  \end{cases}
\end{align*}

Calls to the public inference APIs for the models used a temperature of 0.1. Repeated sampling (N=3) with majority vote was used in both the verification and refutation processes. Statistical significance was assessed using paired t-tests (\texttt{ttest\_rel} from the \texttt{scipy.stats} Python package) to compare performance metrics between different model and prompt combinations. The experiments were conducted in the first half of May 2025 using calls to the public inference APIs for each of the models. 

\subsection{Experimental results}

\begin{table}[!htbp]
\centering
\tiny
\begin{tabular}{llllll}
\toprule
\textbf{Judge Model} & \textbf{Prompt} & \textbf{Macro F1} & \textbf{Coverage} & \textbf{Time (s)} & \textbf{Tokens} \\
\midrule
Claude 3.5 Sonnet & direct & 0.712 (0.023) & 0.748 (0.02) &  34.22 (6.08) & 2914.80 (778.99) \\
& zero & 0.738 (0.029) & 0.542 (0.024) &  53.91 (6.41) & 4863.48 (731.27) \\
& few & 0.716 (0.028) & 0.54 (0.023) &  52.92 (6.41) & 8079.40 (745.71) \\
\midrule
Claude 3.5 Haiku & direct & 0.578 (0.027) & 0.778 (0.02) &  30.52 (4.09) & 2641.07 (704.48) \\
& zero & 0.648 (0.034) & 0.412 (0.023) &  43.12 (3.60) & 4221.73 (685.85) \\
& few & 0.604 (0.034) & 0.438 (0.024) &  47.44 (4.69) & 7673.44 (700.47) \\
\midrule
Llama 4 Maverick & direct & \textbf{0.774} (0.021) & \textbf{0.852} (0.016) &  65.91 (52.53) & 6225.19 (1526.74) \\
& zero & 0.765 (0.025) & 0.618 (0.022) &  75.95 (44.90) & 7492.54 (1357.27) \\
& few & 0.751 (0.023) & 0.805 (0.018) &  65.08 (41.43) & 9945.70 (1444.02) \\
\midrule
Llama 4 Scout & direct & 0.702 (0.025) & 0.712 (0.021) &  63.24 (28.92) & 5403.00 (1833.85)  \\
& zero & 0.712 (0.031) & 0.46 (0.024) &  93.43 (51.27) & 7062.05 (1430.96) \\
& few & 0.694 (0.027) & 0.642 (0.022) &  69.27 (37.27) & 9540.89 (1362.86) \\
\midrule
GPT-4o & direct & 0.592 (0.027) & 0.705 (0.021) &   \textbf{5.61} (8.41) & \textbf{1133.92} (556.51) \\
& zero & 0.603 (0.035) & 0.48 (0.022) &  36.29 (9.44) & 6041.24 (1256.13) \\
& few & 0.69 (0.031) & 0.518 (0.023) &  42.01 (18.48) & 8662.78 (1398.93) \\
\midrule
GPT-4o Mini & direct & 0.536 (0.028) & 0.69 (0.022) &  16.25 (12.91) & 2863.86 (1716.43)  \\
& zero & 0.4 (0.025) & 0.438 (0.023) &  32.88 (7.54) & 5851.67 (959.41) \\
& few & 0.428 (0.028) & 0.488 (0.023) &  47.99 (15.90) & 8787.62 (1120.93) \\
\bottomrule
\end{tabular}
\caption{Performance metrics for $\zeta$ using different judge models and evaluation prompts on GPQA question/answer pairs (N=400).}
\label{tab:finalperfgpqa}
\end{table}

\begin{table}[!htbp]
\centering
\tiny
\begin{tabular}{llllll}
\toprule
\textbf{Judge Model} & \textbf{Prompt} & \textbf{Macro F1} & \textbf{Coverage} & \textbf{Time (s)} & \textbf{Tokens} \\
\midrule
Claude 3.5 Sonnet & direct & 0.667 (0.021) & 1.0 (0.0) &  14.52 (2.92) & 1360.21 (386.44) \\
& zero & 0.664 (0.022) & 1.0 (0.0) &  19.29 (2.80) & 2161.70 (369.97) \\
& few & 0.66 (0.022) & 1.0 (0.0) &  23.64 (2.85) & 5055.46 (382.93) \\
\midrule
Claude 3.5 Haiku & direct & 0.533 (0.023) & 1.0 (0.0) &  14.25 (2.23) & 1346.54 (377.10) \\
& zero & 0.556 (0.024) & 1.0 (0.0) &  19.70 (2.44) & 2070.33 (328.06) \\
& few & 0.577 (0.023) & 0.998 (0.002) &  19.48 (1.55) & 4829.19 (327.03) \\
\midrule
Llama 4 Maverick & direct & \textbf{0.722} (0.021) & \textbf{1.0} (0.0) &  26.83 (21.83) & 2919.78 (825.91) \\
& zero & 0.694 (0.02) & 0.998 (0.002) &  25.20 (10.44) & 3462.93 (665.47) \\
& few & 0.717 (0.021) & 0.99 (0.005) &  24.49 (11.06) & 5756.76 (602.70) \\
\midrule
Llama 4 Scout & direct & 0.636 (0.023) & 1.0 (0.0) &  30.07 (19.93) & 2467.57 (1096.58) \\
& zero & 0.577 (0.024) & 0.998 (0.002) &  24.74 (18.10) & 2689.86 (1180.22) \\
& few & 0.568 (0.023) & 1.0 (0.0) &  24.38 (18.08) & 4955.30 (1124.60) \\
\midrule
GPT-4o & direct & 0.572 (0.023) & 1.0 (0.0) &   \textbf{2.92} (7.85) & \textbf{546.73} (276.40) \\
& zero & 0.497 (0.024) & 1.0 (0.0) &   3.89 (4.92) & 1059.73 (276.40) \\
& few & 0.484 (0.024) & 1.0 (0.0) &   8.00 (8.50) & 3446.44 (343.78) \\
\midrule
GPT-4o Mini & direct & 0.543 (0.023) & 1.0 (0.0) &   9.52 (6.61) & 1543.24 (848.31) \\
& zero & 0.444 (0.021) & 1.0 (0.0) &   6.60 (6.96) & 1722.51 (945.13) \\
& few & 0.538 (0.024) & 1.0 (0.0) &  11.41 (2.60) & 4949.93 (477.92) \\
\bottomrule
\end{tabular}
\caption{Performance metrics for unilateral factuality evaluation using different judge models and evaluation prompts on GPQA question/answer pairs (N=400).}
\label{tab:finalperfgpqauni}
\end{table}

\begin{table}[!htbp]
\centering
\tiny
\begin{tabular}{llllll}
\toprule
\textbf{Judge Model} & \textbf{Prompt} & \textbf{Macro F1} & \textbf{Coverage} & \textbf{Time (s)} & \textbf{Tokens} \\
\midrule
Claude 3.5 Sonnet & direct & 0.81 (0.025) & 0.502 (0.024) &  19.22 (4.49) & 1190.43 (174.00) \\
& zero & 0.733 (0.027) & 0.615 (0.022) &  43.35 (6.44) & 3444.22 (171.12) \\
& few & 0.654 (0.031) & 0.65 (0.022) &  43.19 (5.45) & 6704.68 (161.86) \\
\midrule
Claude 3.5 Haiku & direct & 0.667 (0.033) & 0.458 (0.024) &  15.81 (3.07) & 1014.98 (149.24) \\
& zero & 0.673 (0.036) & 0.385 (0.022) &  38.92 (3.14) & 3095.57 (131.72) \\
& few & 0.653 (0.033) & 0.45 (0.023) &  40.01 (4.09) & 6435.11 (158.40) \\
\midrule
Llama 4 Maverick & direct & 0.673 (0.026) & \textbf{0.768} (0.02) &  21.84 (13.56) & 2008.57 (754.99) \\
& zero & 0.746 (0.029) & 0.528 (0.024) &  36.17 (10.85) & 4421.40 (448.00) \\
& few & 0.692 (0.026) & 0.715 (0.021) &  41.29 (31.55) & 6665.58 (512.60) \\
\midrule
Llama 4 Scout & direct & 0.58 (0.031) & 0.572 (0.023) &  17.02 (9.05) & 1392.07 (434.73) \\
& zero & 0.677 (0.038) & 0.358 (0.022) &  43.41 (16.12) & 4049.26 (380.54) \\
& few & 0.558 (0.031) & 0.632 (0.023) &  36.90 (27.20) & 6267.85 (437.70) \\
\midrule
GPT-4o & direct & 0.67 (0.026) & 0.74 (0.021) &   5.11 (5.49) & \textbf{477.22} (60.34) \\
& zero & 0.738 (0.032) & 0.44 (0.024) &  22.18 (3.49) & 3720.07 (314.42) \\
& few & \textbf{0.833} (0.026) & 0.482 (0.024) &  21.00 (4.86) & 6079.94 (292.95) \\
\midrule
GPT-4o Mini & direct & 0.604 (0.027) & 0.718 (0.021) &   \textbf{2.53} (0.72) & 483.86 (68.15) \\
& zero & 0.525 (0.038) & 0.378 (0.023) &  23.75 (10.23) & 3812.08 (829.15) \\
& few & 0.586 (0.034) & 0.472 (0.023) &  30.69 (16.02) & 6298.49 (254.71) \\
\bottomrule
\end{tabular}
\caption{Performance metrics for $\zeta$ using different judge models and evaluation prompts on SimpleQA question/answer pairs (N=400).}
\label{tab:finalperfsimpleqa}
\end{table}

\begin{table}[!htbp]
\centering
\tiny
\begin{tabular}{llllll}
\toprule
\textbf{Judge Model} & \textbf{Prompt} & \textbf{Macro F1} & \textbf{Coverage} & \textbf{Time (s)} & \textbf{Tokens} \\
\midrule
Claude 3.5 Sonnet & direct & \textbf{0.705} (0.022) & \textbf{1.0} (0.0) &   6.34 (1.59) & 453.53 (80.98) \\
& zero & 0.682 (0.022) & 1.0 (0.0) &  16.28 (5.05) & 1499.60 (112.18) \\
& few & 0.661 (0.023) & 1.0 (0.0) &  16.82 (1.99) & 4331.51 (89.63) \\
\midrule
Claude 3.5 Haiku & direct & 0.595 (0.023) & 1.0 (0.0) &   6.69 (1.60) & 489.73 (91.76) \\
& zero & 0.55 (0.025) & 1.0 (0.0) &  16.29 (3.87) & 1505.90 (93.60) \\
& few & 0.523 (0.024) & 1.0 (0.0) &  17.02 (1.55) & 4332.01 (63.46) \\
\midrule
Llama 4 Maverick & direct & 0.643 (0.023) & 1.0 (0.0) &   6.71 (4.20) & 814.69 (350.94) \\
& zero & 0.663 (0.023) & 0.992 (0.004) &  14.88 (8.78) & 2097.99 (230.58) \\
& few & 0.648 (0.023) & 0.992 (0.004) &  16.97 (10.19) & 4524.26 (265.16) \\
\midrule
Llama 4 Scout & direct & 0.578 (0.023) & 1.0 (0.0) &   5.73 (4.96) & 511.43 (286.82) \\
& zero & 0.559 (0.025) & 1.0 (0.0) &   8.62 (8.44) & 1192.79 (516.58) \\
& few & 0.552 (0.024) & 1.0 (0.0) &   5.08 (6.68) & 3306.05 (376.05) \\
\midrule
GPT-4o & direct & 0.62 (0.023) & 1.0 (0.0) &   2.55 (5.42) & \textbf{220.40} (30.07) \\
& zero & 0.614 (0.024) & 1.0 (0.0) &   3.49 (3.70) & 733.39 (30.07) \\
& few & 0.632 (0.023) & 1.0 (0.0) &   7.38 (9.60) & 3088.39 (30.07) \\
\midrule
GPT-4o Mini & direct & 0.587 (0.023) & 1.0 (0.0) &   \textbf{1.08} (0.19) & 220.58 (30.79) \\
& zero & 0.528 (0.023) & 1.0 (0.0) &   1.56 (1.52) & 788.69 (185.84) \\
& few & 0.572 (0.022) & 1.0 (0.0) &   6.75 (2.37) & 3923.80 (318.41) \\
\bottomrule
\end{tabular}
\caption{Performance metrics for unilateral factuality evaluation using different judge models and evaluation prompts on SimpleQA question/answer pairs (N=400).}
\label{tab:finalperfsimpleqauni}
\end{table}

\begin{table}[!htbp]
\centering
\tiny
\begin{tabular}{llllllllll}
\toprule
\textbf{Judge Model} & \textbf{Prompt} & \textbf{$\langle \mathfrak{t},\mathfrak{t} \rangle$} & \textbf{$\langle \mathfrak{t},\mathfrak{f} \rangle$} & \textbf{$\langle \mathfrak{f},\mathfrak{t} \rangle$} & \textbf{$\langle \mathfrak{f},\mathfrak{f} \rangle$} \\
\midrule
Claude 3.5 Sonnet & direct & 0.202 (0.018) & 0.392 (0.022) & 0.355 (0.022) & 0.05 (0.011) \\
& zero & 0.435 (0.023) & 0.218 (0.02) & 0.325 (0.021) & 0.022 (0.007)  \\
& few & 0.448 (0.023) & 0.18 (0.018) & 0.36 (0.022) & 0.012 (0.005) \\
\midrule
Claude 3.5 Haiku & direct & 0.11 (0.015) & 0.53 (0.023) & 0.248 (0.02) & 0.112 (0.015) \\
& zero & 0.53 (0.023) & 0.208 (0.019) & 0.205 (0.019) & 0.058 (0.011)  \\
& few & 0.502 (0.024) & 0.212 (0.02) & 0.225 (0.02) & 0.06 (0.011) \\
\midrule
Llama 4 Maverick & direct & 0.04 (0.009) & 0.442 (0.023) & 0.41 (0.023) & 0.108 (0.015) \\
& zero & 0.37 (0.022) & 0.245 (0.02) & 0.372 (0.023) & 0.012 (0.005) \\
& few & 0.155 (0.017) & 0.438 (0.023) & 0.368 (0.022) & 0.04 (0.009) \\
\midrule
Llama 4 Scout & direct & 0.072 (0.013) & 0.368 (0.023) & 0.345 (0.022) & 0.215 (0.019) \\
& zero & 0.525 (0.024) & 0.245 (0.021) & 0.215 (0.019) & 0.015 (0.006) \\
& few & 0.33 (0.022) & 0.385 (0.024) & 0.258 (0.02) & 0.028 (0.007) \\
\midrule
GPT-4o & direct & 0.082 (0.013) & 0.358 (0.022) & 0.348 (0.023) & 0.212 (0.018) \\
& zero & 0.502 (0.023) & 0.108 (0.015) & 0.372 (0.022) & 0.018 (0.006) \\
& few & 0.468 (0.024) & 0.155 (0.017) & 0.362 (0.022) & 0.015 (0.006) \\
\midrule
GPT-4o Mini & direct & 0.172 (0.018) & 0.145 (0.017) & 0.545 (0.023) & 0.138 (0.016) \\
& zero & 0.555 (0.023) & 0.018 (0.006) & 0.42 (0.023) & 0.008 (0.004) \\
& few & 0.51 (0.023) & 0.028 (0.007) & 0.46 (0.022) & 0.002 (0.002) \\
\bottomrule
\end{tabular}
\caption{Truth value probabilities for $\zeta$ using different judge models and evaluation prompts for GPQA.}
\label{tab:finaltvgpqa}
\end{table}

\begin{table}[!htbp]
\centering
\tiny
\begin{tabular}{llllllllll}
\toprule
\textbf{Judge Model} & \textbf{Prompt} & \textbf{$\langle \mathfrak{t},\mathfrak{t} \rangle$} & \textbf{$\langle \mathfrak{t},\mathfrak{f} \rangle$} & \textbf{$\langle \mathfrak{f},\mathfrak{t} \rangle$} & \textbf{$\langle \mathfrak{f},\mathfrak{f} \rangle$} \\
\midrule
Claude 3.5 Sonnet & direct & 0.052 (0.01) & 0.218 (0.02) & 0.285 (0.021) & 0.445 (0.024) \\
& zero & 0.245 (0.02) & 0.152 (0.016) & 0.462 (0.022) & 0.14 (0.016) \\
& few & 0.278 (0.021) & 0.118 (0.014) & 0.532 (0.023) & 0.072 (0.012) \\
\midrule
Claude 3.5 Haiku & direct & 0.035 (0.008) & 0.282 (0.022) & 0.175 (0.018) & 0.507 (0.024) \\
& zero & 0.148 (0.016) & 0.145 (0.016) & 0.24 (0.02) & 0.468 (0.023) \\
& few & 0.132 (0.016) & 0.162 (0.017) & 0.288 (0.021) & 0.418 (0.023) \\
\midrule
Llama 4 Maverick & direct & 0.088 (0.013) & 0.588 (0.022) & 0.18 (0.018) & 0.145 (0.016) \\
& zero & 0.438 (0.024) & 0.35 (0.022) & 0.178 (0.018) & 0.032 (0.008) \\
& few & 0.218 (0.019) & 0.525 (0.023) & 0.19 (0.019) & 0.068 (0.012) \\
\midrule
Llama 4 Scout & direct & 0.125 (0.015) & 0.368 (0.022) & 0.205 (0.018) & 0.302 (0.021) \\
& zero & 0.62 (0.022) & 0.245 (0.02) & 0.112 (0.015) & 0.022 (0.007) \\
& few & 0.232 (0.02) & 0.53 (0.023) & 0.102 (0.014) & 0.132 (0.016) \\
\midrule
GPT-4o & direct & 0.058 (0.011) & 0.462 (0.023) & 0.278 (0.022) & 0.202 (0.019) \\
& zero & 0.56 (0.024) & 0.105 (0.014) & 0.335 (0.023) & 0.0 (0.0) \\
& few & 0.51 (0.024) & 0.205 (0.018) & 0.278 (0.021) & 0.008 (0.004) \\
\midrule
GPT-4o Mini & direct & 0.145 (0.017) & 0.228 (0.019) & 0.49 (0.023) & 0.138 (0.016) \\
& zero & 0.62 (0.023) & 0.055 (0.01) & 0.322 (0.022) & 0.002 (0.002) \\
& few & 0.488 (0.022) & 0.272 (0.02) & 0.2 (0.018) & 0.038 (0.009) \\
\bottomrule
\end{tabular}
\caption{Truth value probabilities for $\zeta$ using different judge models and evaluation prompts for SimpleQA.}
\label{tab:finaltvsimpleqa}
\end{table}

\begin{table}[!htbp]
\centering
\tiny
\renewcommand{\arraystretch}{0.9}
\begin{tabular}{llllll}
\toprule
\textbf{Topic} & \textbf{Model} & $\langle \mathfrak{t}, \mathfrak{f} \rangle$ & $\langle \mathfrak{f}, \mathfrak{t} \rangle$ & $\langle \mathfrak{t}, \mathfrak{t} \rangle$ & $\langle \mathfrak{f}, \mathfrak{f} \rangle$ \\
\midrule
Art & GPT-4o & 0.127 (0.039) & 0.327 (0.056) & 0.545 (0.059) & 0.000 (0.000) \\
 & GPT-4o-mini & 0.036 (0.022) & 0.291 (0.054) & 0.673 (0.056) & 0.000 (0.000) \\
 & Claude 3.5 Sonnet & 0.182 (0.045) & 0.509 (0.059) & 0.182 (0.045) & 0.127 (0.039) \\
 & Claude 3.5 Haiku & 0.200 (0.047) & 0.218 (0.048) & 0.109 (0.037) & 0.473 (0.060) \\
 & Llama 4 Maverick & 0.345 (0.055) & 0.145 (0.041) & 0.491 (0.058) & 0.018 (0.016) \\
 & Llama 4 Scout & 0.273 (0.051) & 0.055 (0.026) & 0.655 (0.054) & 0.018 (0.016) \\
\midrule
Geography & GPT-4o & 0.114 (0.042) & 0.409 (0.063) & 0.477 (0.067) & 0.000 (0.000) \\
 & GPT-4o-mini & 0.114 (0.042) & 0.273 (0.058) & 0.614 (0.064) & 0.000 (0.000) \\
 & Claude 3.5 Sonnet & 0.159 (0.047) & 0.500 (0.065) & 0.273 (0.060) & 0.068 (0.032) \\
 & Claude 3.5 Haiku & 0.182 (0.049) & 0.318 (0.059) & 0.159 (0.049) & 0.341 (0.061) \\
 & Llama 4 Maverick & 0.386 (0.062) & 0.182 (0.050) & 0.409 (0.064) & 0.023 (0.019) \\
 & Llama 4 Scout & 0.295 (0.059) & 0.091 (0.037) & 0.614 (0.063) & 0.000 (0.000) \\
\midrule
History & GPT-4o & 0.143 (0.076) & 0.357 (0.108) & 0.500 (0.109) & 0.000 (0.000) \\
 & GPT-4o-mini & 0.071 (0.056) & 0.500 (0.111) & 0.429 (0.112) & 0.000 (0.000) \\
 & Claude 3.5 Sonnet & 0.214 (0.091) & 0.357 (0.108) & 0.214 (0.094) & 0.214 (0.088) \\
 & Claude 3.5 Haiku & 0.214 (0.090) & 0.214 (0.094) & 0.143 (0.080) & 0.429 (0.108) \\
 & Llama 4 Maverick & 0.429 (0.111) & 0.214 (0.093) & 0.357 (0.108) & 0.000 (0.000) \\
 & Llama 4 Scout & 0.357 (0.107) & 0.143 (0.079) & 0.500 (0.109) & 0.000 (0.000) \\
\midrule
Music & GPT-4o & 0.000 (0.000) & 0.361 (0.069) & 0.639 (0.069) & 0.000 (0.000) \\
 & GPT-4o-mini & 0.028 (0.023) & 0.278 (0.066) & 0.694 (0.068) & 0.000 (0.000) \\
 & Claude 3.5 Sonnet & 0.056 (0.033) & 0.611 (0.068) & 0.194 (0.056) & 0.139 (0.049) \\
 & Claude 3.5 Haiku & 0.111 (0.044) & 0.278 (0.065) & 0.083 (0.039) & 0.528 (0.072) \\
 & Llama 4 Maverick & 0.278 (0.065) & 0.167 (0.054) & 0.556 (0.072) & 0.000 (0.000) \\
 & Llama 4 Scout & 0.194 (0.056) & 0.167 (0.054) & 0.611 (0.072) & 0.028 (0.023) \\
\midrule
Other & GPT-4o & 0.136 (0.044) & 0.341 (0.062) & 0.523 (0.065) & 0.000 (0.000) \\
 & GPT-4o-mini & 0.068 (0.032) & 0.455 (0.065) & 0.477 (0.067) & 0.000 (0.000) \\
 & Claude 3.5 Sonnet & 0.136 (0.045) & 0.455 (0.064) & 0.205 (0.052) & 0.205 (0.052) \\
 & Claude 3.5 Haiku & 0.091 (0.038) & 0.227 (0.055) & 0.136 (0.045) & 0.545 (0.067) \\
 & Llama 4 Maverick & 0.295 (0.060) & 0.114 (0.041) & 0.523 (0.064) & 0.068 (0.033) \\
 & Llama 4 Scout & 0.295 (0.060) & 0.136 (0.044) & 0.455 (0.064) & 0.114 (0.041) \\
\midrule
Politics & GPT-4o & 0.131 (0.038) & 0.344 (0.052) & 0.525 (0.053) & 0.000 (0.000) \\
 & GPT-4o-mini & 0.066 (0.028) & 0.295 (0.051) & 0.623 (0.055) & 0.016 (0.014) \\
 & Claude 3.5 Sonnet & 0.180 (0.043) & 0.443 (0.054) & 0.197 (0.044) & 0.180 (0.042) \\
 & Claude 3.5 Haiku & 0.164 (0.041) & 0.246 (0.048) & 0.148 (0.041) & 0.443 (0.055) \\
 & Llama 4 Maverick & 0.410 (0.055) & 0.262 (0.050) & 0.295 (0.052) & 0.033 (0.020) \\
 & Llama 4 Scout & 0.246 (0.048) & 0.148 (0.039) & 0.574 (0.055) & 0.033 (0.020) \\
\midrule
Sci/Tech & GPT-4o & 0.101 (0.031) & 0.304 (0.047) & 0.595 (0.050) & 0.000 (0.000) \\
 & GPT-4o-mini & 0.038 (0.018) & 0.291 (0.046) & 0.671 (0.048) & 0.000 (0.000) \\
 & Claude 3.5 Sonnet & 0.228 (0.043) & 0.329 (0.047) & 0.367 (0.050) & 0.076 (0.026) \\
 & Claude 3.5 Haiku & 0.152 (0.036) & 0.114 (0.032) & 0.228 (0.042) & 0.506 (0.049) \\
 & Llama 4 Maverick & 0.456 (0.049) & 0.203 (0.039) & 0.329 (0.047) & 0.013 (0.011) \\
 & Llama 4 Scout & 0.253 (0.043) & 0.127 (0.033) & 0.620 (0.049) & 0.000 (0.000) \\
\midrule
Sports & GPT-4o & 0.161 (0.059) & 0.258 (0.068) & 0.581 (0.078) & 0.000 (0.000) \\
 & GPT-4o-mini & 0.065 (0.038) & 0.387 (0.073) & 0.548 (0.073) & 0.000 (0.000) \\
 & Claude 3.5 Sonnet & 0.065 (0.038) & 0.452 (0.081) & 0.290 (0.071) & 0.194 (0.062) \\
 & Claude 3.5 Haiku & 0.129 (0.051) & 0.226 (0.066) & 0.194 (0.062) & 0.452 (0.075) \\
 & Llama 4 Maverick & 0.290 (0.071) & 0.161 (0.056) & 0.452 (0.078) & 0.065 (0.036) \\
 & Llama 4 Scout & 0.129 (0.051) & 0.097 (0.045) & 0.774 (0.063) & 0.000 (0.000) \\
\midrule
TV & GPT-4o & 0.040 (0.033) & 0.400 (0.082) & 0.560 (0.082) & 0.000 (0.000) \\
 & GPT-4o-mini & 0.000 (0.000) & 0.400 (0.084) & 0.600 (0.084) & 0.000 (0.000) \\
 & Claude 3.5 Sonnet & 0.080 (0.047) & 0.640 (0.082) & 0.080 (0.045) & 0.200 (0.069) \\
 & Claude 3.5 Haiku & 0.000 (0.000) & 0.520 (0.086) & 0.040 (0.033) & 0.440 (0.086) \\
 & Llama 4 Maverick & 0.120 (0.054) & 0.160 (0.063) & 0.600 (0.081) & 0.120 (0.056) \\
 & Llama 4 Scout & 0.160 (0.061) & 0.080 (0.047) & 0.760 (0.073) & 0.000 (0.000) \\
\midrule
Games & GPT-4o & 0.000 (0.000) & 0.182 (0.096) & 0.818 (0.096) & 0.000 (0.000) \\
 & GPT-4o-mini & 0.091 (0.073) & 0.091 (0.072) & 0.818 (0.094) & 0.000 (0.000) \\
 & Claude 3.5 Sonnet & 0.000 (0.000) & 0.455 (0.122) & 0.455 (0.125) & 0.091 (0.073) \\
 & Claude 3.5 Haiku & 0.182 (0.098) & 0.273 (0.116) & 0.091 (0.071) & 0.455 (0.128) \\
 & Llama 4 Maverick & 0.182 (0.097) & 0.000 (0.000) & 0.818 (0.097) & 0.000 (0.000) \\
 & Llama 4 Scout & 0.182 (0.095) & 0.000 (0.000) & 0.818 (0.095) & 0.000 (0.000) \\
\bottomrule
\end{tabular}
\caption{Truth value probablities by model and SimpleQA topic with zero-shot prompting. Values show proportions (SE) for each of the following epistemic states: $\langle t, f \rangle$ (verified, not refuted), $\langle f, t \rangle$ (not verified, refuted), $\langle t, t \rangle$ (contradictory), $\langle f, f \rangle$ (uncertain). Standard errors estimated via subsampling~\citep{politis1994large}.}
\label{tab:bilateral-topic-distribution}
\end{table}

\end{document}